\documentclass{article}

\usepackage[final,nonatbib]{neurips_2020}
\usepackage[utf8]{inputenc} 
\usepackage[T1]{fontenc}    
\usepackage{hyperref}       
\usepackage{url}            
\usepackage{booktabs}       
\usepackage{amsfonts}       
\usepackage{nicefrac}       
\usepackage{thm-restate}
\usepackage{microtype}      
\usepackage{xcolor,enumitem}
\usepackage{graphicx}
\usepackage{amsmath,amsfonts,bm}
\usepackage{amssymb,amsthm}
\renewcommand{\[}{\begin{eqnarray}}
\renewcommand{\]}{\end{eqnarray}}

\DeclareMathOperator*{\E}{\mathbb{E}}

\newcommand{\sgn}{\textnormal{sgn}}

\usepackage{mathbbol}

\newtheorem{lemma}{Lemma}

\newtheorem{remark}{Remark}
\usepackage{wrapfig}
\usepackage{caption}

\title{On Infinite-Width Hypernetworks}

\author{%
Etai Littwin\thanks{Equal Contribution} \\
School of Computer Science \\
Tel Aviv University \\
Tel Aviv, Israel \\
\texttt{etai.littwin@gmail.com} \\
\And
Tomer Galanti$^*$ \\
School of Computer Science \\
Tel Aviv University \\
Tel Aviv, Israel \\
\texttt{tomerga2@tauex.tau.ac.il} \\
\AND
Lior Wolf \\
School of Computer Science \\
Tel Aviv University \\
Tel Aviv, Israel \\
\texttt{wolf@cs.tau.ac.il} 
\And
Greg Yang \\
Microsoft Research AI \\
\texttt{gregyang@microsoft.com} \\
}

\begin{document}

\maketitle

\begin{abstract}
{\em Hypernetworks} are architectures that produce the weights of a task-specific {\em primary network}. A notable application of hypernetworks in the recent literature involves learning to output functional representations. In these scenarios, the hypernetwork learns a representation corresponding to the weights of a shallow MLP, which typically encodes shape or image information. While such representations have seen considerable success in practice, they remain lacking in the theoretical guarantees in the wide regime of the standard architectures. In this work, we study wide over-parameterized hypernetworks. We show that unlike typical architectures, infinitely wide hypernetworks do not guarantee convergence to a global minima under gradient descent. We further show that convexity can be achieved by increasing the dimensionality of the hypernetwork's output, to represent wide MLPs. In the dually infinite-width regime, we identify the functional priors of these architectures by deriving their corresponding GP and NTK kernels, the latter of which we refer to as the {\em hyperkernel}. As part of this study, we make a mathematical contribution by deriving tight bounds on high order Taylor expansion terms of standard fully connected ReLU networks.
\end{abstract}



\section{Introduction}
\begin{wrapfigure}{r}{0.28\textwidth}
\includegraphics[width=0.28\textwidth]{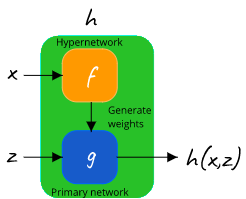}
\end{wrapfigure}
In this work, we analyze the training dynamics of over-parameterized meta networks, which are networks that output the weights of other networks, often referred to as {\em hypernetworks}.
In the typical framework, a function $h$ involves two networks, $f$ and $g$. The \emph{hypernetwork} $f$ takes the input $x$ (typically an image) and returns the weights of the \emph{primary network}, $g$, which then takes the input $z$ and returns the output of $h$.

The literature of hypernetworks is roughly divided into two main categories. {In the functional representation literature~\cite{Littwin_2019_ICCV,sitzmann2020implicit,wu2018pay,9054042,Klocek_2019} the input to the hypernetwork $f$ is typically an image. For shape reconstruction tasks, the  network $g$ represents the shape via a signed distance field, where the input are coordinates in 3D space. In image completion tasks, the inputs to $g$ are image coordinates, and the output is the corresponding pixel intensity.} {In these settings, $f$ is typically a large network and $g$ is typically a shallow fully connected network.}  

In the second category~\cite{brock2018smash,lorraine2018stochastic,Oswald2020Continual,zhang2018graph}, hypernetworks are typically used for hyper-parameter search, where $x$ is being treated as a hyperparameter descriptor and is optimized alongside with the network's weights. In this paper, we consider models corresponding to the first group of methods.



Following a prominent thread in the recent literature, our study takes place in the regime of wide networks.~\cite{jacotNTK} recently showed that, when the width of the network approaches infinity, the gradient-descent training dynamics of a fully connected network $f$ can be characterized by a kernel, called the \emph{Neural Tangent Kernel} (or NTK for short). {In other words, as the width of each layer approaches infinity,
provided with proper scaling and initialization of the weights, it holds that:
\begin{equation}
\frac{\partial f(x;w)}{\partial w} \cdot \frac{\partial^{\top} f(x';w)}{\partial w} \to \Theta^f(x,x')
\end{equation}
as the width, $n$ of $f$ tends to infinity. Here, $w$ are the weights of the network $f(x;w)$. As shown in~\cite{jacotNTK}, as the width tends to infinity, when minimizing the squared loss using gradient descent, the evolution through time of the function computed by the network follows the dynamics of kernel gradient descent with kernel $\Theta^f$.} To prove this phenomenon, various papers~\cite{NIPS2019_9063,Bai2020TaylorizedTT,Bai2020BeyondLO} introduce a Taylor expansion of the network output around the point of initialization and consider its values.
It is shown that the first-order term is deterministic during the SGD optimization and the higher-order terms converge to zero as the width $n$ tends to infinity. 


A natural question that arises when considering hypernetworks is whether a similar ``wide'' regime exists, where trained and untrained networks may be functionally approximated by kernels. If so, since this architecture involves two networks, the ``wide'' regime needs a more refined definition, taking into account both networks.

Our contributions:
\begin{enumerate}[leftmargin=*]
    \item We show that infinitely wide hypernetworks can induce highly non-convex training dynamics under gradient descent. The complexity of the optimization problem is highly dependent on the architecture of the primary network $g$, which may considerably impair the trainability of the architecture if not defined appropriately.
    \item However, when the widths of both the hypernetwork $f$ and the primary network $g$ tend to infinity, the optimization dynamics of the {hypernetwork simplifies}, and its neural tangent kernel (which we call the \emph{hyperkernel}) has a well defined infinite-width limit governing the network evolution.
    \item We verify our theory empirically and also demonstrate the utility of this \emph{hyperkernel} on several {functional representation} tasks. Consistent with prior observations on kernel methods, {the hypernetwork induced kernels} also outperforms a trained hypernetwork when training data is small. 
    \item We make a technical contribution by deriving asymptotically tight bounds on high order Taylor expansion terms in ReLU MLPs. Our result partially settles a conjecture posed in \cite{Dyer2020Asymptotics} regarding the asymptotic behavior of general correlation functions.
\end{enumerate}

\subsection{Related Works}

\paragraph{Hypernetworks} 
Hypernetworks were first introduced under this name in~\cite{Ha2017HyperNetworks}, are networks that generate the weights of a second {\em primary} network that computes the actual task. However, the idea of having one network predict the weights of another was proposed earlier and has reemerged multiple times~\cite{klein2015dynamic,7410424,jia2016dynamic}. 
{The tool can naturally be applied for image representations tasks. 
In~\cite{Littwin_2019_ICCV}, they applied hypernetworks for 3D shape reconstruction from a single image. In~\cite{sitzmann2020implicit} hypernetworks were shown to be useful for learning shared image representations. Hypernetworks were also shown to be effective in non-image domains. For instance, hypernetworks achieve state of the art results on the task of decoding error correcting codes~\cite{NIPS2019_8504}.
}

{Several publications consider a different framework, in which, the inputs $x$ of the hypernetwork are optimized alongside to the weights of the hypernetwork.} In this setting, hypernetworks were recently used for continuous learning by~\cite{Oswald2020Continual}.  Hypernetworks can be efficiently used for neural architecture search, as was demonstrated by~\cite{brock2018smash,zhang2018graph}, where a feedforward regression (with network $f$) replaces direct gradient-based learning of the weights of the primary network while its architecture is being explored. Lorraine et al. applied hypernetworks for hyperparameters selection~\cite{lorraine2018stochastic}.

Despite their success and increasing prominence, little theoretical work was done in order to better
understand hypernetworks and their behavior. A recent paper~\cite{jayakumar2020multiplicative} studies the role of multiplicative interaction within a unifying framework to describe a range of classical and modern neural network architectural motifs, such as gating, attention layers,
hypernetworks, and dynamic convolutions amongst others. It is shown that standard neural networks are a strict subset of neural networks with multiplicative interactions. {In~\cite{galanti2020modularity} the authors theoretically study the modular properties of hypernetworks. In particular, they show that compared to standard embedding methods, hypernetworks are exponentially more expressive when the primary network is of small complexity. In this work, we provide a complementary perspective and show that a shallow primary network is a requirement for successful training. }\cite{chang2020principled} showed that applying standard initializations on a hypernetwork produces sub-optimal initialization of the primary network. A principled technique for weight initialization in hypernetworks is then developed.

\paragraph{Gaussian Processes and Neural Tangent Kernel}
The connection between infinitely wide neural networks, Gaussian processes and kernel methods, has been the focus of many recent papers~\cite{jacotNTK,lee2018deep,Yang2019ScalingLO,NIPS2017_6879,Schoenholz2016DeepIP,Rudner2018OnTC,Woodworth2020KernelAR,Wei2018OnTM,Novak2018BayesianCN}. Empirical support has demonstrated the power of CNTK (convolutional neural tangent kernel) on popular datasets, demonstrating new state of the art results for kernel methods~\cite{conf/nips/AroraDH0SW19,li2019enhanced}. \cite{littwin2020random} showed that ReLU ResNets~\cite{He2016DeepRL} can have NTK convergence occur even when depth and width simultaneously tend to infinity, provided proper initialization. In this work, we extend the kernel analysis of networks to hypernetworks, and characterize the regime in which the kernels converge and training dynamics simplify.

\section{Setup}

In this section, we introduce the setting of the analysis considered in this paper. We begin by defining fully connected neural networks and hypernetworks in the context of the NTK framework. 

\noindent{\bf Neural networks \quad} In the NTK framework, a fully connected neural network, $f(x;w) = y^L(x)$, is defined in the following manner: 
\begin{equation}
\begin{cases}  
y^{l}(x) = \sqrt{\frac{1}{n_{l-1}}} W^l q^{l-1}(x) \\
q^{l}(x) = \sqrt{2} \cdot \sigma(y^l(x))
\end{cases} 
\textnormal{ and } q^0(x) = x\, ,
\end{equation}
where $\sigma : \mathbb{R} \to \mathbb{R}$ is the activation function of $f$. Throughout the paper, we specifically take $\sigma$ to be a piece-wise linear function with a finite number of pieces (e.g., the ReLU activation $\textnormal{ReLU}(x) := \max(0,x)$ and the Leaky ReLU activation $\textnormal{ReLU}_{\alpha}(x) = 
\begin{cases} 
x  &\textnormal{ if } x \geq 0\\
\alpha x &\textnormal{ if } x<0
\end{cases} 
$). The weight matrices $W^l \in \mathbb{R}^{n_l \times n_{l-1}}$ are trainable variables, initialized independently according to a standard normal distribution, $W^l_{i,j} \sim \mathcal{N}(0,1)$. The width of $f$ is denoted by $n := \min(n_1,\dots,n_{L-1})$. The parameters $w$ are aggregated as a long vector $w = (vec(W^1),\dots,vec(W^L))$. The coefficients $\sqrt{1/n_{l-1}}$ serve for normalizing the activations of each layer. This parametrization is nonstandard, and we will refer to it as the NTK parameterization. It has already been employed in several recent works~\cite{karras2018progressive,DBLP:journals/corr/Laarhoven17b,48223}. For simplicity, in many cases, we will omit to specify the weights $w$ associated with our model.

\noindent{\bf Hypernetworks \quad} Given the input tuple $u = (x,z) \in \mathbb{R}^{n_0 + m_0}$, {we consider models of the form: $h(u;w) := g(z;f(x;w))$, where $f(x;w)$  and $g(z;v)$ are two neural network architectures with depth $L$ and $H$ respectively. The function $f(x;w)$ referred to as {\em hypernetwork}, takes the input $x$ and computes the weights $v = f(x;w)$ of a second neural network $g(z;v)$, referred as the {\em primary network}, which is assumed to output a scalar.} As before, the variable $w$ stands for a vector of trainable parameters ($v$ is not trained directly and is given by $f$).

We parameterize the primary network $g(z;v) = g^{H}(z;v) $ as follows:
\begin{equation}
\begin{cases}  
g^l(z;v) = \sqrt{\frac{1}{m_{l-1}}} V^l \cdot a^{l-1}(z;v) \\
a^{l}(z;v) = \sqrt{2} \cdot \phi(g^l(z;v))
\end{cases} 
\textnormal{ and } a^0(z) = z
\end{equation}
Here, the weights of the primary network $V^l(x) \in \mathbb{R}^{m_l \times m_{l-1}}$ are given in a concatenated vector form by the output of the hypernetwork $f(x;w) =  v = (vec(V^1),\dots,vec(V^{H}))$. The output dimension of the hypernetwork $f$ is therefore $n_{L} = \sum^{H}_{i=1} m_{i} \cdot m_{i-1}$. We denote by $f^d(x;w) := V^d(x;w) := V^d$ the $d$'th output matrix of $f(x;w)$. The width of $g$ is denoted by $m := \min(m_1,\dots,m_{H-1})$. The function $\phi$ is an element-wise continuously differentiable function or a piece-wise linear function with a finite number of pieces.

\noindent{\bf Optimization \quad} Let $S = \{(u_i,y_i)\}^{N}_{i=1}$, where $u_i = (x_i,z_i)$ be some dataset and let $\ell(a,b) := |a-b|^p$ be the $\ell^p$-loss function. For a given hypernetwork $h(u;w)$, we are interested in selecting the parameters $w$ that minimize the empirical risk:
\begin{equation}
c(w) 
:= \sum^{N}_{i=1} \ell(h(u_i;w), y_i)
\end{equation}
For simplicity, oftentimes we will simply write $\ell_{i}(a) := \ell(a,y_i)$ and $h_i(w) := h(u_i) := h(u_i;w)$, depending on the context. In order to minimize the empirical error $c(w)$, we consider the SGD method with learning rate $\mu > 0$ and step of the form $w_{t+1} \leftarrow w_{t} - \mu \nabla_w \ell_{j_t}(h_{j_t}(w_t))$ {for some index $j_t \sim U[N]$ that is selected uniformly at random for the $t$'th iteration}. A continuous version of the GD method is the gradient flow method, in which $\dot{w} = -\mu \nabla_{w} c(w)$. In recent works~\cite{karras2018progressive,NIPS2019_9063,conf/nips/AroraDH0SW19,DBLP:journals/corr/Laarhoven17b}, the optimization dynamics of the gradient method for standard fully-connected neural networks was analyzed, as the network width tends to infinity. In our work, since hypernetworks consist of two interacting neural networks, there are multiple ways in which the size can tend to infinity. We consider two cases: (i) the width of $f$ tends to infinity and that of $g$ is fixed and (ii) the width of both $f$ and $g$ tend to infinity. 


\section{Dynamics of Hypernetworks}\label{sec:init}

\paragraph{Infinitely wide $f$ without infinitely wide $g$ induces non-convex optimization}
In the NTK literature, it is common to adopt a functional view of the network evolution by analyzing the dynamics of the output of the network along with the cost, typically a convex function, as a function of that output. 
In the hypernetwork case, this presents us with two possible viewpoints of the same optimization problem of $h(u)=g(z;f(x))$. On one hand, since only the hypernetwork $f$ contains the trainable parameters, we can view the optimization of $h$ under the loss $\ell$ as training of $f$ under the loss $\ell \circ g$.
The classical NTK theory would imply that $f$ evolves linearly when its width tends to infinity, but because $\ell \circ g$ is in general not convex anymore, even when $\ell$ originally is, an infinitely wide $f$ without an infinitely wide $g$ does not guarantee convergence to a global optimum.
In what follows, we make this point precise by characterising how nonlinear the dynamics becomes in terms of the depth of $g$.

After a single stochastic gradient descent step with learning rate $\mu$, the hypernetwork output for example $i$ is given by $h_i\big(w - \mu \nabla_w \ell_j\big)$. When computing the Taylor approximation around $w$ with respect to the function $h$ at the point $w' = w - \mu \nabla_w \ell_j$, it holds that:
\begin{equation}
\begin{aligned}
h_i\big(w - \mu \nabla_w \ell_j\big) 
= \sum_{r=0}^\infty \frac{1}{r!}\langle \nabla^{(r)}_w h_i, (-\mu \nabla_w \ell_j)^{r} \rangle  = \sum_{r=0}^\infty \frac{1}{r!}\left(-\mu\frac{\partial \ell_j}{\partial h_j}\right)^r \cdot \mathcal K^{(r)}_{i,j}
\end{aligned}
\label{eq:taylorexpansion}
\end{equation}
where $\mathcal{K}^{(r)}_{i,j} := \langle \nabla^{(r)}_w h_i, (\nabla_w h_j)^{r} \rangle$,
and $\nabla^{(r)}_w h_i$ is the $r$ tensor that holds the $r$'th derivative of the output $h_i$. {The terms $\langle \nabla^{(r)}_w h_i, (-\mu \nabla_w \ell_j)^{r} \rangle$ are the multivariate extensions of the Taylor expansion terms $h^{(r)}_i(w) (w'-w)^r$, and take the general form of correlation functions as introduced in Eq.~5 in the appendix. This equation holds for neural networks with smooth activation functions (including hypernetworks), and holds in approximation for piece-wise linear activation functions.}

Previous works have shown that, if $h$ is a wide fully connected network, the first order term ($r=1$) converges to the NTK, while higher order terms ($r>1$) scale like $\mathcal{O}(1/\sqrt{n})$~\cite{NIPS2019_9063,Dyer2020Asymptotics}.
Hence, for large widths and small learning rates, these higher order terms vanish, and the loss surface appears deterministic and linear at initialization, and remains so during training.

However, the situation is more complex for hypernetworks.
As shown in the following theorem, {for infinitely wide hypernetworks and finite primary network}, the behaviour depends on the depth and width of the generated primary network. Specifically, when the primary network is deep and narrow, the higher order terms in Eq.~\ref{eq:taylorexpansion} may not vanish, and parameter dynamics can be highly non-convex.

\begin{restatable}[Higher order terms for hypernetworks]{theorem}{inithypernet}\label{thm:hyper}
Let $h(u) = g(z;f(x))$ for a hypernetwork $f$ and an primary network $g$. Then, we have:
\begin{equation}
\mathcal{K}^{(r)}_{i,j}  \sim \begin{cases}
     n^{H-r} & \text{if $r > H$}\\
     1 & \text{otherwise.}
\end{cases}
\end{equation}
\end{restatable}
Thm.~\ref{thm:hyper} illustrates the effect of the depth of the primary network $g$ on the evolution of the output $h$. The larger $H$ is, the more non-linear the evolution is, even when $f$ is infinitely wide. Indeed, we observe empirically that when $f$ is wide and kept fixed, a deeper $g$ incurs slower training, and lower overall test performance as illustrated in Fig.~\ref{fig:var}. 

As a special case of this theorem, when taking $H=1$, we can also derive the asymptotic behaviour of $\mathcal{K}^{(r)}_{i,j} \sim n^{1-r}$ for a neural network $h$. This provides a tighter bound than the previously conjectured $\mathcal{O}(1/n)$ upper bound~\cite{Dyer2020Asymptotics}. The following remark is a consequence of this result and is validated in the supplementary material.

\begin{remark}
The $r$'th order term of the Taylor expansion in Eq.~\ref{eq:taylorexpansion} is of order $\mathcal{O}(\frac{\mu^{r}}{r! \cdot n^{r-1}})$ instead of the previously postulated $\mathcal{O}(\frac{\mu^{r}}{r! \cdot n})$. Therefore, it is evident that for any choice $\mu = o(\sqrt{n})$, all of the high order terms tend to zero as $n \to \infty$. This is opposed the previous bound, which guarantees that all of the high order terms tend to zero as $n \to \infty$ only when $\mu$ is constant. 
\end{remark}

\section{Dually Infinite Hypernetworks}\label{sec:infinite}

{It has been shown by~\cite{jacotNTK,lee2018deep} that NNGPs and neural tangent kernels fully characterise the training dynamics of infinitely wide networks. As a result, in various publications~\cite{littwin2020random,Hanin2020Finite}, these kernels are being treated as functional priors of neural networks.} In the previous section, we have shown that the Taylor expansion of the hypernetwork is non-linear when the size of the primary network is finite. In this section, we consider the case when \emph{both} hyper and primary networks are infinitely wide, with the intention of gaining insight into the functional prior of wide hypernetworks. For this purpose, we draw a formal correspondence between infinitely wide hypernetworks and GPs, and use this connection to derive the corresponding neural tangent kernel. 



\subsection{The NNGP kernel}

Previous work have shown the equivalence between popular architectures, and Gaussian processes, when the width of the architecture tends to infinity. This equivalence has sparked renewed interest in kernel methods, through the corresponding NNGP kernel, and the Neural Tangent Kernel (NTK) induced by the architecture, which fully characterise the training dynamics of infinitely wide networks. This equivalence has recently been unified to encompass most architectures which use a pre-defined set of generic computational blocks \cite{Yang2019ScalingLO,Yang2019TensorPI}. Hypernetworks represent a different class of neural networks where the parameters contain randomly initialized matrices except the last layer whose parameters are aggregated as a rank 3 tensor. All of the matrices/tensors dimensions tend to infinity.
This means the results of \cite{Yang2019ScalingLO,Yang2019TensorPI} do not apply to hypernetworks.
Nevertheless, by considering sequential limit taking, where we take the limit of the width of $f$ ahead of the width of $g$, we show the output of $f$ achieves a GP behaviour, essentially feeding $g$ with Gaussian distributed weights with adaptive variances.
A formal argument is presented in the following theorem.

\begin{figure}[ht]
    \centering
    \begin{tabular}{cc}
    \includegraphics[width=.40\linewidth]{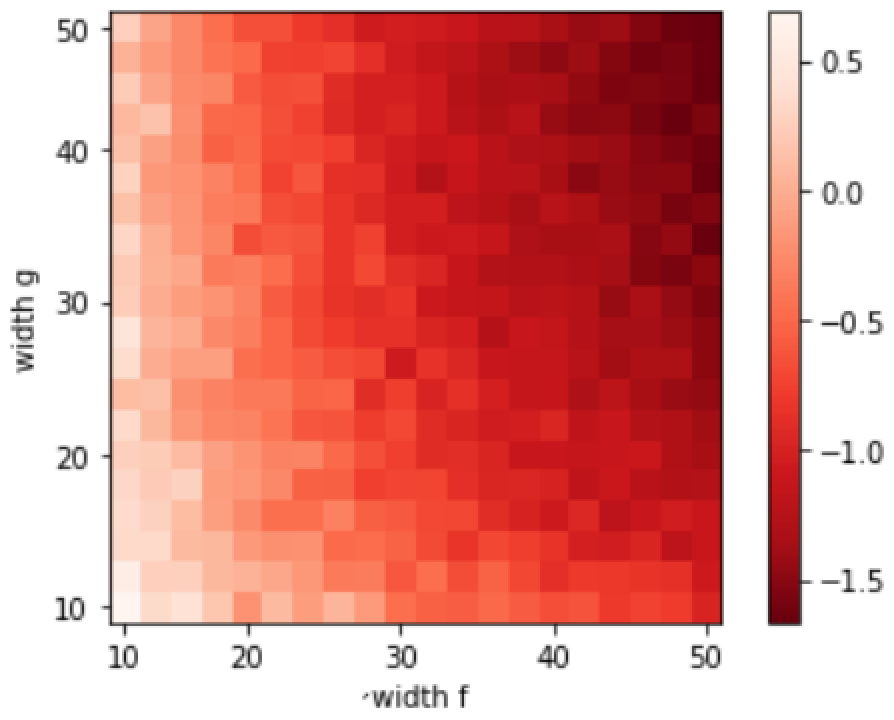}&
    \includegraphics[width=.40\linewidth]{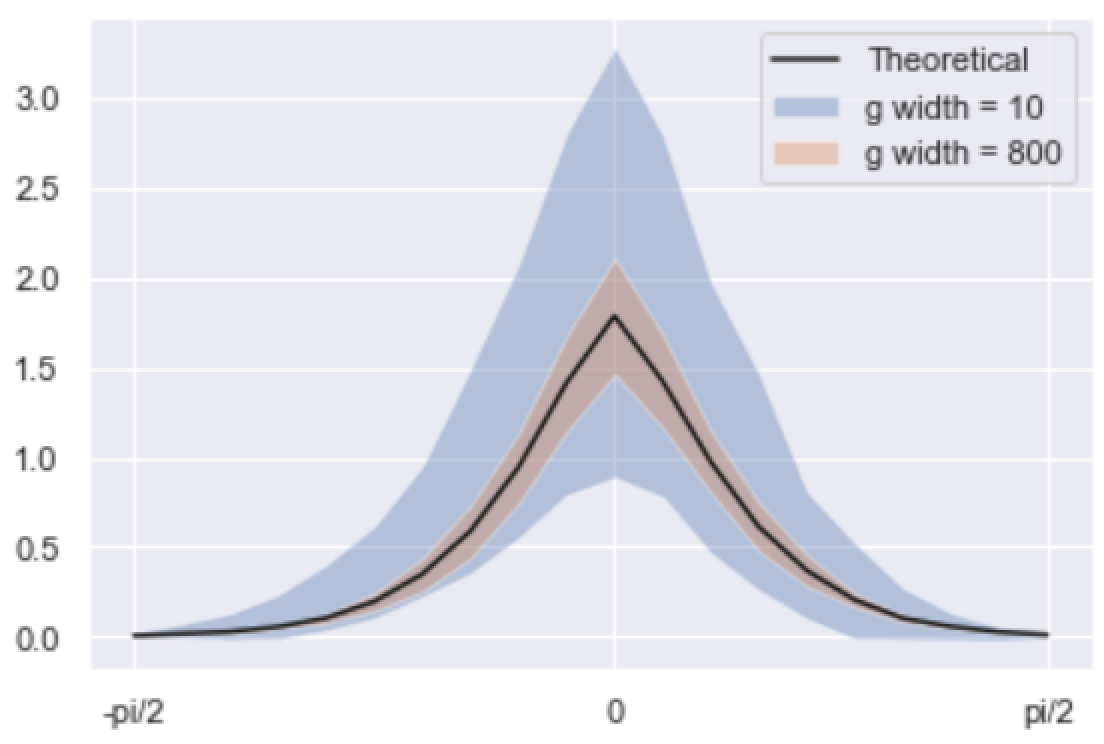}\\
    (a)&(b)
    \\
    \end{tabular}
    \caption{{\bf Convergence to the hyperkernel.}
    \textbf{(a)} Empirical variance of kernel values in log-scale for a single entry for varying width $f$ and $g$. Variance of the kernel converges to zero only when the widths of $f$ and $g$ both increase.
    \textbf{(b)} Empirical kernel value for $z = (1,0), z' = (\cos(\theta),\sin(\theta))$, and $x = x' =  (1,0)$ for different values of $\theta \in [-\frac{\pi}{2},\frac{\pi}{2}]$. Convergence to a deterministic kernel is observed only when both $f$ and $g$ are wide.} 
    \label{fig:kernel}
\end{figure}

\begin{restatable}[Hypernetworks as GPs]{theorem}{GP}\label{thm:gp}
Let $h(u) = g(z;f(x))$ be a hypernetwork. For any pair of inputs $u = (x,z)$ and $u' = (x',z')$, let  $\Sigma^0(z,z') = \frac{z^\top z'}{m_0}, S^0(x,x') = \frac{x^\top x'}{n_0}$. Then, it holds for any unit $i$ in layer $0<l\leq H$ of the primary network:
\begin{equation}
g^l_i(z;f(x)) \stackrel{d}{\longrightarrow} \mathcal{G}_i^l(u)
\end{equation}
as $m,n \to \infty$ sequentially. Here, $\{\mathcal{G}_i^l(u)\}_{i=1}^{m_l}$ are independent Gaussian processes, such that, $(\mathcal{G}_i^l(u), \mathcal{G}_i^l(u')) \sim \mathcal{N}\big(0,\Lambda^l(u,u')\big)$ defined by the following recursion:
\begin{align}\label{mf}
\Lambda^{l+1}(u,u') = 
\begin{pmatrix}
\Sigma^{l}(u,u) & 
\Sigma^{l}(u',u)  \\
\Sigma^{l}(u,u') & 
\Sigma^{l}(u',u')  
\end{pmatrix}\bigodot \begin{pmatrix}
S^L(x,x)& 
S^L(x',x) \\
S^L(x,x')& 
S^L(x',x')  
\end{pmatrix} 
\end{align}
\begin{equation}\label{s}
\Sigma^{l}(u,u') =  2\E_{(u,v)\sim \mathcal{N}(0,\Lambda^l)}[\sigma(u)\cdot \sigma(v)]
\end{equation}
where $S^L(x,x')$ is defined recursively:
\begin{align}
S^{l}(x,x') = 2\E_{(u,v)\sim \mathcal{N}(0,\Gamma^l)}[\sigma(u)\cdot \sigma(v)] ~\textnormal{ and }~
\Gamma^l(x,x') = \begin{pmatrix}
S^{l}(x,x) & 
S^{l}(x',x)  \\
S^{l}(x,x') & 
S^{l}(x',x')  
\end{pmatrix}
\end{align}
\end{restatable}

In other words, the NNGP kernel, governing the behaviour of wide untrained hypernetworks, is given by the Hadamard product of the GP kernels of $f$ and $g$ (see Eq.~\ref{mf}). 

As a consequence of the above theorem, we observe that the NNGP kernel of $h$ at each layer, $\Lambda^l(u,u')$, is simply a function of $\Sigma^0(z,z'), S^0(x,x')$.

\begin{restatable}{corollary}{compGP}
Let $h(u) = g(z;f(x))$ be a hypernetwork.
For any $0< l\leq H$, there exists a function $\mathcal{F}^l$, such that, for all pairs of inputs $u = (x,z)$ and $u' = (x',z')$, it holds that:
\begin{equation}
\Lambda^H(u,u') = \mathcal{F}\left(\Sigma^0(z,z'),S^0(x,x')\right)
\end{equation}
\end{restatable}
 
The factorization of the NNGP kernel into a function of $\Sigma^0(z,z')$ and $S^0(x,x')$ provides a convenient way to explicitly encode useful invariances into the kernel.

{As an example, in the following remark, we investigate the behaviour of the NNGP kernel of $h$, when the inputs $z$ are preprocessed random random Fourier features as suggested by~\cite{NIPS2007_3182,tancik2020Fourier}. 

\begin{restatable}{remark}{periodicGP}\label{remark:perGP}
Let $p(z) = [\cos(W^1_i z + b^1_i)]^{k}_{i=1}$ be a Fourier features preprocessing, where $W^1_{i,j} \sim \mathcal{N}(0,1)$ and biases $b_{i} \sim U[-\pi,\pi]$. Let $h(u) = g(p(z);f(x))$ be a hypernetwork, with $z$ preprocessed according to $p$. Let $u = (x,z)$ and $u'=(x',z')$ be two pairs of inputs. Then,  $\Lambda^l(u,u')$ is a function of $\exp[-\|z-z'\|^2_2/2]$ and $S^L(x,x')$. 
\end{restatable}

The above remark shows that for any given inputs $x,x'$, the NNGP kernel depends on $z,z'$ only through the distance between $z$ and $z'$, which has been shown to be especially useful in implicit neural representation~\cite{tancik2020Fourier}.}


We next derive the corresponding neural tangent kernel of hypernetworks, referred to as hyperkernels.

\begin{figure*}[ht]
    \centering
    \begin{tabular}{ccc}
        \includegraphics[width=.3\linewidth]{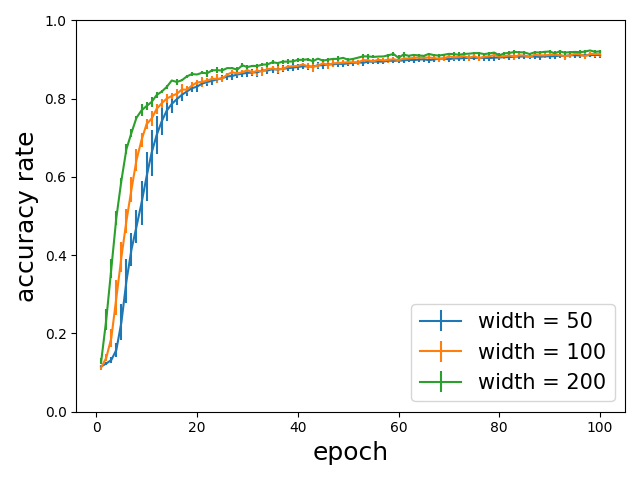} & 
    \includegraphics[width=.3\linewidth]{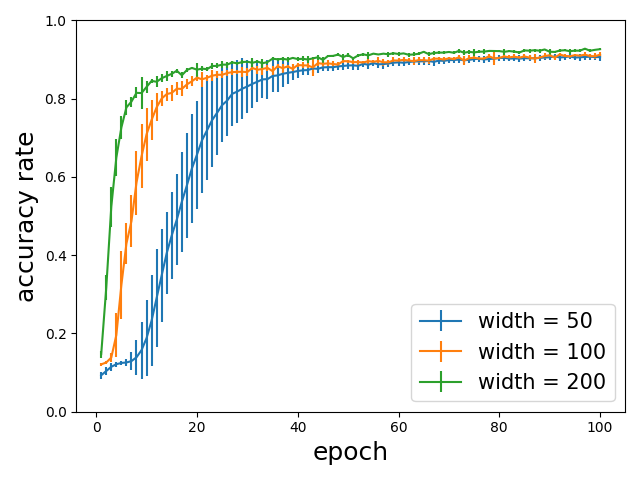} &
    \includegraphics[width=.3\linewidth]{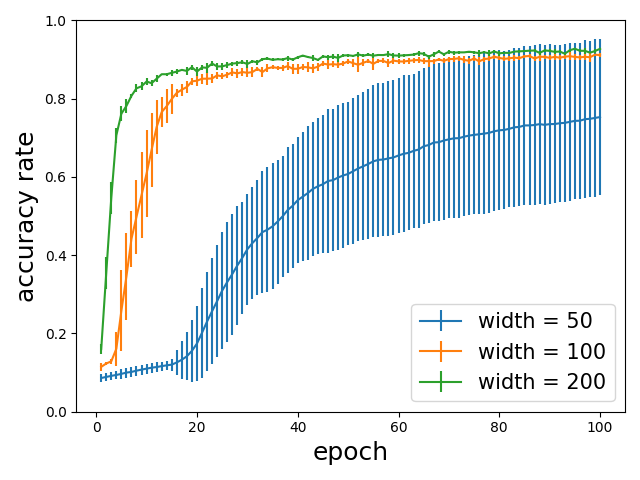} \\
    \includegraphics[width=.3\linewidth]{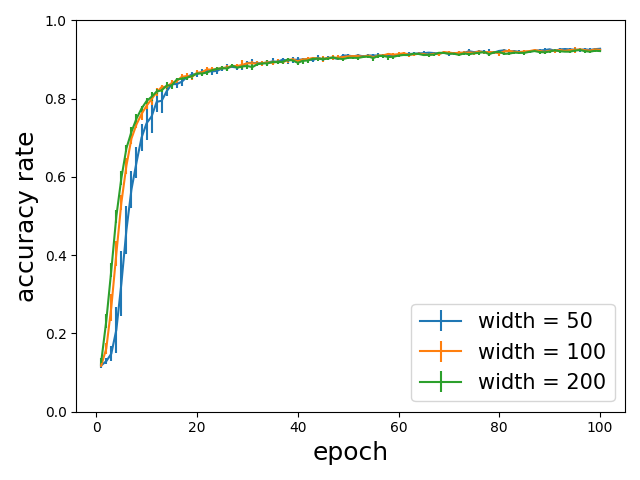} & 
    \includegraphics[width=.3\linewidth]{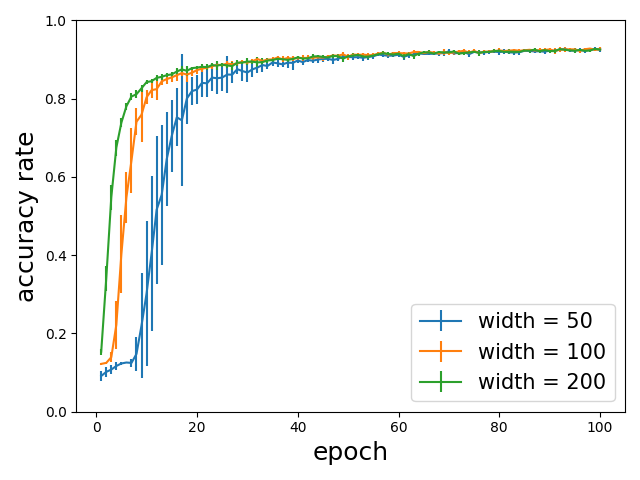} &
    \includegraphics[width=.3\linewidth]{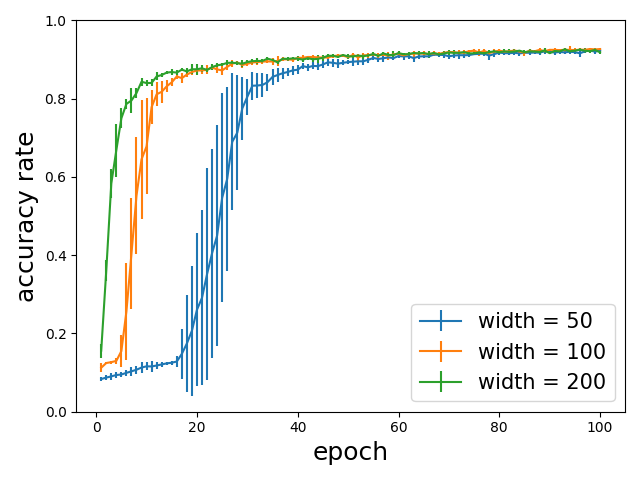}  \\
    (a) & (b) & (c)\\
    \end{tabular}
    \caption{
    \textbf{A hypernetwork with a wider and \emph{shallower} primary network $g$ trains faster and achieves better test performance} on the MNIST {\bf(top)} and CIFAR10 {\bf(bottom)} rotations prediction task.
    We fix the hypernetwork $f$ and the depth of $g$ at {\bf (a)} 3, {\bf (b)} 6 and {\bf (c)} 8, while varying the width $m$ of $g$. The x-axis specifies the epoch and the y-axis the accuracy at test time.  }
    \label{fig:var}
\end{figure*}



\subsection{The Hyperkernel}\label{sec:hyperkernel}

Recall the definition of the NTK as the infinite width limit of the Jacobian inner product given by:
\begin{equation}
\begin{aligned}
\mathcal{K}^h(u,u') = \frac{\partial h(u)}{\partial w}\cdot \frac{\partial^{\top} h(u')}{\partial w} = \frac{\partial g(z;f(x))}{\partial f(x)}\cdot \mathcal{K}^f(x,x') \cdot \frac{\partial^{\top} g(z';f(x'))}{\partial f(x')}
\end{aligned}
\end{equation}
where $\mathcal{K}^f(x,x') := \frac{\partial f(x)}{\partial w}\cdot \frac{\partial^{\top} f(x')}{\partial w}$ and $\mathcal{K}^g(u,u') := \frac{\partial g(z;f(x))}{\partial f(x)}\cdot \frac{\partial^{\top} g(z';f(x'))}{\partial f(x')}$. In the following theorem we show that $\mathcal{K}^h(u,u')$ converges in probability at initialization to a limiting kernel in the sequentially infinite width limit of $f$ and $g$, denoted by $\Theta^h(u,u')$. 
Furthermore, we show that the hyperkernel is decomposed to the Hadamard product between the kernels corresponding to $f$ and $g$. In addition, we show that the derivative of the hyperkernel with respect to time tends to zero at initialization.

\begin{restatable}[Hyperkernel decomposition and convergence at initialization]{theorem}{hyperkernel}\label{thm:hyperkernel}
Let $h(u;w) = g(z;f(x;w))$ be a hypernetwork. Then,
\begin{equation}
\mathcal{K}^h(u,u') \stackrel{p}{\longrightarrow} \Theta^h(u,u') 
\end{equation}
where:
\begin{equation}\label{eq:comp}
\Theta^h(u,u') = \Theta^f(x,x') \cdot \Theta^g(u,u',S^L(x,x'))
\end{equation}
such that:
\begin{equation}
\mathcal{K}^f(x,x') \stackrel{p}{\longrightarrow} \Theta^f(x,x')\cdot I \textnormal{ and } \mathcal{K}^g(u,u') \stackrel{p}{\longrightarrow} \Theta^g(u,u',S^L(x,x'))
\end{equation}
moreover, if $w$ evolves throughout gradient flow, we have:
\begin{equation}
\frac{\partial \mathcal{K}^h(u,u')}{\partial t}\Big\vert_{t=0} \stackrel{p}{\longrightarrow} 0
\end{equation}
where the limits are taken with respect to $m,n\to \infty$ sequentially.
\end{restatable}


{As a consequence of Thm.~\ref{thm:hyperkernel}, when applying a Fourier features preprocessing to $z$, one obtains that $\Theta^g(u,u')$ becomes shift invariant. 

\begin{restatable}{remark}{periodicNTK}\label{remark:perNTK}
Let $p(z)$ be as in Remark~\ref{remark:perGP}. Let $h(u) = g(p(z);f(x))$ be a hypernetwork, where $z$ is preprocessed according to $p$. Let $u=(x,z)$ and $u'=(x',z')$ be two pairs of inputs. Then, $\Theta^g(u,u')$ is a function of $\exp[-\|z-z'\|^2_2/2]$ and $S^0(x,x')$.
\end{restatable}
Note that $\Theta^f$ is the standard limiting NTK of $f$ and depends only on the inputs $\{x_i\}_{i=1}^N$. However from Eq.~\ref{mf}, the term $\Theta^g$ requires the computation of the NNGP kernel of $f$ in advance in order to compute the terms $\{\Sigma^l,\dot{\Sigma}^l\}$. This form provides an intuitive factorization of the hyperkernel into a term $\Theta^f$ which depends on the meta function and data, and $\Theta^g$ which can be though of as a conditional term.

}

\section{Experiments}

{Our experiments are divided into two main parts. In the first part, we validate the ideas presented in our theoretical analysis and study the effect of the width and depth of $g$ on the optimization of a hypernetwork. In the second part, we evaluate the performance of the NNGP and NTK kernels on image representation tasks. For further implementation details on all experiments see Appendix~A.} 

\begin{table*}[ht]
\centering
\begin{tabular}{{lllll|llll}}
\hline
             & \multicolumn{4}{l|}{Representation} & \multicolumn{4}{l}{Inpainting} \\\hline
\multicolumn{1}{l|}{$N$}  &50       &100         &200        &500        &50       &100         &200        &500       \\\hline
\multicolumn{1}{l|}{HK}  &0.055       &0.050         &0.043        &0.032        &0.057       &0.051         &0.047        &0.038       \\
\multicolumn{1}{l|}{NNGP}         &0.051        &0.045         &0.037        &0.026        &0.054        &0.047         &0.043        &0.034       \\
\multicolumn{1}{l|}{HN} &0.12        &0.08         &0.052        &0.041        &0.16      &0.098        &0.066       &0.49 \\   \hline  
\end{tabular}
\caption{
    {\textbf{Results on image representation and inpainting.} Reported are the MSE of the reconstructed image on test set where $N$ is the number of training samples. As can be seen, in low data regime the kernels outperform a trained hypernetwork on both tasks, and the NNGP consistently outperforms the rest.}}
    \label{tab:implicit}
\end{table*}

\subsection{Convergence of the Hyperkernel}

We verified our results of Thm.~\ref{thm:hyperkernel} by constructing a simple hypernetwork, for which both $f$ and $g$ are four layered fully connected networks with ReLU activations. For the input of $f$, we used a fixed 2D vector $x = (1,-1)$. The input $z(\theta)$ of $g$ varied according to $z(\theta) = (\sin(\theta),\cos(\theta))$, where $\theta \in [-\frac{\pi}{2},\frac{\pi}{2}]$. We then compute the empirical hyperkernel as follows, while varying the width of both $f$ and $g$:
\begin{equation}
\mathcal{K}^h(u,u') = \nabla_w h(x,z(\theta))\cdot \nabla_w^\top h(x,z(\theta))
\end{equation}
Results are presented in Fig.~\ref{fig:kernel}. As can be seen, convergence to a fixed kernel is only observed in the dually wide regime, as stated in Thm. \ref{thm:hyperkernel}.

\subsection{Training Dynamics}

We consider a rotation prediction task. In this task, the hypernetwork $f$ is provided with a randomly rotated image $x$ and the primary network $g$ is provided with a rotated version $z$ of $x$ with a random angle $\alpha$. The setting is cast into a classification task, where the goal is to predict the closest value to $\alpha/360$ within $\{\alpha_i = 30i/360 \mid i = 0,\dots,11\}$. We experimented with the MNIST~\cite{lecun-mnisthandwrittendigit-2010} and CIFAR10~\cite{Krizhevsky10convolutionaldeep} datasets. For each dataset we took $10000$ training samples only.

We investigate the effect of the depth and width of $g$ on the training dynamics of a hypernetwork. We compared the performance of hypernetworks of various architectures to investigate the effect of the depth and width of $g$. The architectures of the hypernetwork and the primary network are as follows. The hypernetwork, $f$, is a fully-connected ReLU neural network of depth $4$ and width $200$. The inputs of $f$ are flattened vectors of dimension $c \cdot h^2$, where $c$ specifies the number of channels and $h$ the height/width of each image ($c=1$ and $h=28$ for MNIST and $c=3$ and $h=32$ for CIFAR10). The primary network $g$ is a fully-connected ReLU neural network of depth $\in \{3,6,8\}$. Since the MNIST rotations dataset is simpler, we varied the width of $g$ in $\in \{10,50,100\}$ and for the the CIFAR10 variation we selected the width of $g$ to be $\in \{100,200,300\}$. The network outputs $12$ values and is trained using the cross-entropy loss. 

We trained the hypernetworks for 100 epochs, using the SGD method with batch size 100 and learning rate $\mu=0.01$. For completeness, we conducted a sensitivity study on the learning rate for both datasets, to show that the reported behaviour is consistent for any chosen learning rate, see appendix. In Fig.~\ref{fig:var}(a-c) we compare the performance of the various architectures on the MNIST and CIFAR10 rotations prediction tasks. The performance is computed as an average and standard deviation (error bars) over $100$ runs. As can be seen, we observe a clear improvement in test performance as the width of $g$ increases, especially at the initialization. When comparing the three plots, we observe that when $f$ is wide and kept fixed, a deeper $g$ incurs slower training, and lower overall test performance. This is aligned with the conclusions of Thm.~\ref{thm:hyper}. 

\subsection{Image representation and Inpainting}
{We compare the performance of a hypernetwork and kernel regression with the hyperkernel on two visual tasks: functional image representation and inpainting. In the MNIST image representation task, the goal of the hypernetwork $f$ is to represent an input image via the network $g$, which receives image pixel coordinates and outputs pixel values. In the inpainting task, the goal is the same where only half of the image is observed by $f$.
\newpage
\label{sec:imgcomp}
\begin{figure}[ht]
    \centering
    \begin{tabular}{ccc}
    \includegraphics[width=.69\linewidth]{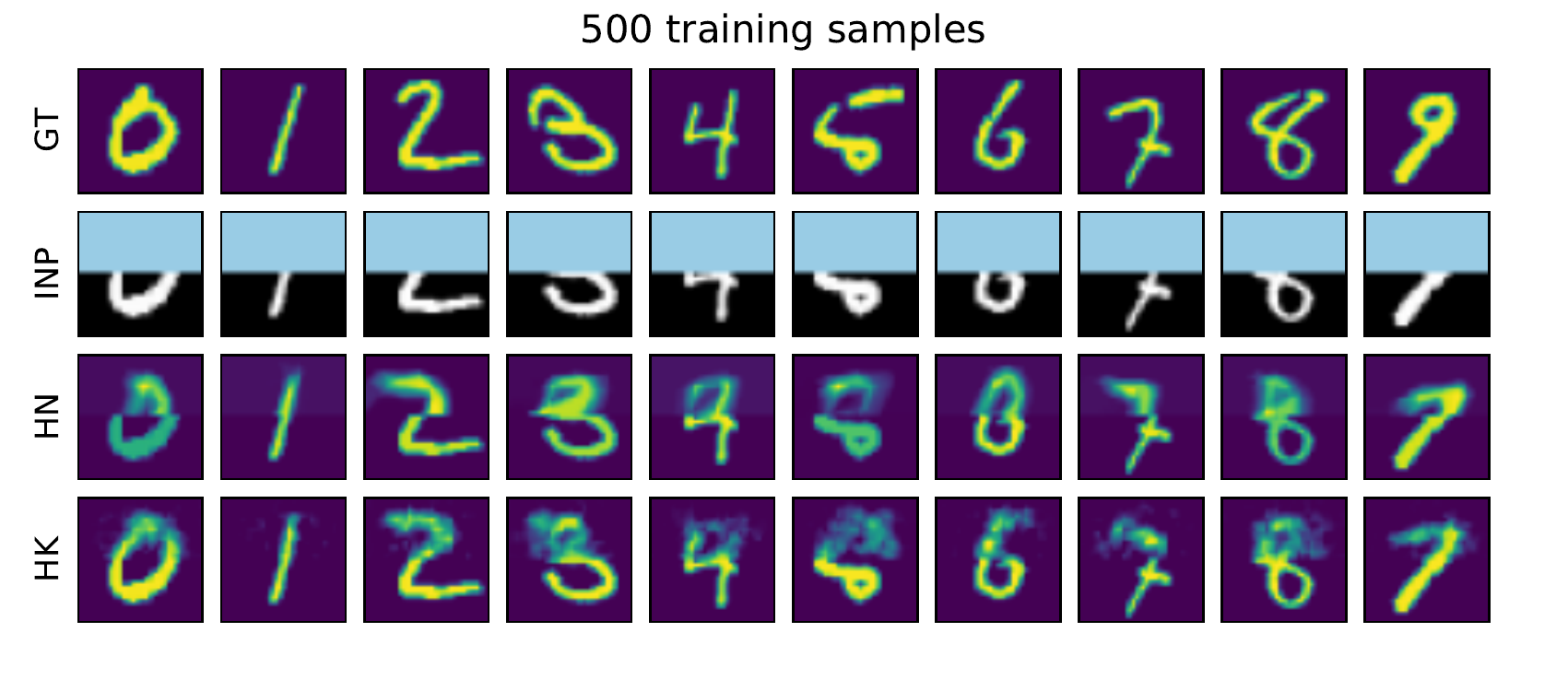} \\
    \end{tabular}
    \caption{{\bf Results on image inpainting.} \textbf{(Row 1)} ground-truth images.
    \textbf{(Row 2)} corresponding inputs of meta-network $f$. 
    \textbf{(Row 3)} reconstruction by the hypernetwork.
    \textbf{(Row 4)} reconstruction by the hyperkernel.
    See Section \ref{sec:imgcomp} for experimental details.}
    \label{fig:completion300}
\end{figure}

\noindent{\bf Problem Setup\quad}
We cast these problems as a meta-learning problem, where $f$ receives an image, and the goal of the primary network $g:[28]^2 \rightarrow [0,1]$ is then to learn a conditional mapping from pixel coordinates to pixel values for all the pixels in the image, with the MSE as the metric. Our training dataset $S = \{(u_i,y_i)\}_{i=1}^N$ then consists of samples $u_i = (x_i,z_i)$, such that, $x_i$ are images, and $z_i$ are random pixel location (i.e., a tuple $\in [28]^2$), and $y_i$ is a label specifying the pixel value at the specified location (normalized between 0 and 1). In both experiments, training was done on randomly sampled training data of varying size, specified by $N$.

\noindent{\bf Evaluation\quad}
We evaluate the performance of both training a hypernetwork, and using kernel regression with the hyperkernel. For kernel regression, we use the following formula to infer the pixel value of a test point $u$:
\begin{equation}
(\Theta^h(u,u_1),..., \Theta^h(u,u_N))\cdot \big(\Theta^h(U,U) + \epsilon\cdot I\big)^{-1}\cdot Y
\end{equation}
where $\Theta^h(U,U) = (\Theta^h(u_i,u_j))_{i,j\in [N]}$ is the hyperkernel matrix evaluated on all of the training data and $Y = (y_i)^{N}_{i=1}$ is the vector of labels in the training dataset and $\epsilon = 0.001$. 

In Tab.~\ref{tab:implicit}, we compare the results of the hyperkernel and the NNGP kernel with the corresponding hypernetwork. The reported numbers are averages over $20$ runs. As can be seen, in the case of a small dataset, the kernels outperforms the hypernetwork, and the NNGP outperforms the rest.} 

\section{Conclusions}



In this paper, we apply the well-established large width analysis to hypernetwork type models. For the class of models analyzed, we have shown that a wide hypernetwork must be coupled with a wide primary network in order achieve a simplified, convex training dynamics as in standard architectures. 
The deeper $g$ is, the more complicated the evolution is. In the dually infinite case, when the widths of both the hyper and primary networks tend to infinity, the optimization of the hypernetwork become convex and is governed by the proposed hyperkernel.
The analysis presented in this paper is limited to a specific type of hypernetworks used in the literature, typically found in the functional neural representation literature, and we leave the extension of this work to additional types of hyper models to future work.\\
Some of the tools developed in this study, also apply for regular NTKs. Specifically,~\cite{Dyer2020Asymptotics} provide a conjecture, for which one of its consequences is that 
$\mathcal{K}^{(r)}_{i,j} = \mathcal{O}(1/n)$. In Thm.~\ref{thm:hyper} we prove that this hypothesized upper bound is increasingly loose as $r$ increases, and prove an asymptotic behaviour in the order of $\mathcal{K}^{(r)}_{i,j} \sim 1/n^{r-1}$.

\clearpage

\section*{Broader Impact}
This work improves our understanding {and design} of hypernetworks and hopefully will help us improve the transparency of machine learning involving them. {Beyond that, this work falls under the category of basic research and does not seem to have particular societal or ethical implications.}

\section*{Acknowledgements and Funding Disclosure}

This project has received funding from the European Research Council (ERC) under the European
Union’s Horizon 2020 research and innovation programme (grant ERC CoG 725974). The contribution of Tomer Galanti is part of Ph.D. thesis research conducted at Tel Aviv University.

\bibliography{arxiv_version}
\bibliographystyle{plain}

\newpage
\appendix
\onecolumn

\section{Implementation Details}

\subsection{Convergence of the Hyperkernel}

In Fig.~1(a) (main text) we plot the variance of the kernel values $\mathcal{K}^h(u,u')$ in log-scale, as a function of the width of both $f$ and $g$. The variance was computed empirically over $k=100$ normally distributed samples $w$. As can be seen, the variance of the kernel tends to zero only when both widths increase. In Fig.~1(b) (main text) we plot the value of $\mathcal{K}^h(u,u')$ and its variance for a fixed hypernetwork $f$ of width $500$ and $g$ of width $10$ or $800$. The $x$-axis specifies the value of $\theta \in [-\frac{\pi}{2},\frac{\pi}{2}]$ and the y-axis specifies the value of the kernel. As can be seen, the expected value of the empirical kernel, $\mathcal{K}^h(u,u')$, is equal to the width-limit kernel (e.g., theoretical kernel) for both widths $10$ and $800$. In addition, convergence of the width-limit kernel is guaranteed only when the widths of both networks increase, highlighting the importance of wide architectures for both the hyper and implicit networks for stable training. 

\subsection{Image Completion and Impainting}
\noindent{\bf Architectures\quad}
In both tasks, we used fully connected architectures, where $f$ contains two hidden layers, and $g$ contains one hidden layer. The hyperkernel used corresponds to the infinite width limit of the same architecture. For the input of $g$, we used random Fourier features~\cite{tancik2020Fourier} of the pixel coordinates as inputs for both the hyperkernel and the hypernetwork. To ease on the computational burden of computing the full kernel matrix $\Theta^h(U,U)$ when evaluating the hyperkernel, we compute smaller kernel matrices on subsets of the data $\{U^s_i\} = \{x_i^s,z_i^s\}_{s\in [10]}$, where each subset contains 1k input images $\{x^s_i\}$, and 20 random image coordinates per input, producing a kernel matrix of size $20k\times20k$. The final output prediction is then given by:  
\begin{equation}
\frac{1}{10}\sum_s(\Theta^h(u,u_1^s),..., \Theta^h(u,u_N^s))\cdot \big(\Theta^h(U^s,U^s) + \epsilon\cdot I\big)^{-1}\cdot Y^s
\end{equation}
where $Y^s$ are the corresponding labels of the subset $U^s$.
For the hypernetwork evaluation, we used the same inputs $\{U^s_i\}_{s\in[10]}$ to train the hypernetwork using a batchsize of $20$, and a learning rate of $0.01$ which was found to produce the best results.  

\section{Additional Experiments}



\subsection{Sensitivity Study}

To further demonstrate the behavior reported in Fig.~2 (main text), 
we verified that it is consistent regardless of the value of the learning rate. We used the same architectures as in the rotations prediction experiments, i.e., $f$ is a fully-connected ReLU neural network of depth $4$ and width $200$ and $g$ is of depth $\in \{3,6,8\}$ and width $\in \{50,100,200\}$. We vary the learning rate: $\mu = 10^{j-7}$, for $j=0,\dots,7$. For each value of the learning rate, we report the average performance (and standard deviation over $100$ runs) of the various architectures after $40$ epochs of training. 

As can be seen in Fig.~\ref{fig:varLr}, when $f$ is wide and kept fixed, there is a clear improvement in test performance as the width of $g$ increases, for every learning rate in which the networks provide non-trivial performance. When $f$ is wide and kept fixed, a deeper $g$ incurs slower training and lower overall test performance. We note that it might seem that the width of $g$ does not affect the performance when the learning rate is $\mu=0.01$ in all settings except Figs.~\ref{fig:varLr}(c,f). Indeed, we can verify from Fig.~2 (main text) that the performance at epoch $40$ is indeed similar for different widths. However, for earlier epochs, the performance improves for shallower and wider architectures.

\begin{figure*}[ht]
    \centering
    \begin{tabular}{ccc}
    \includegraphics[width=.3\linewidth]{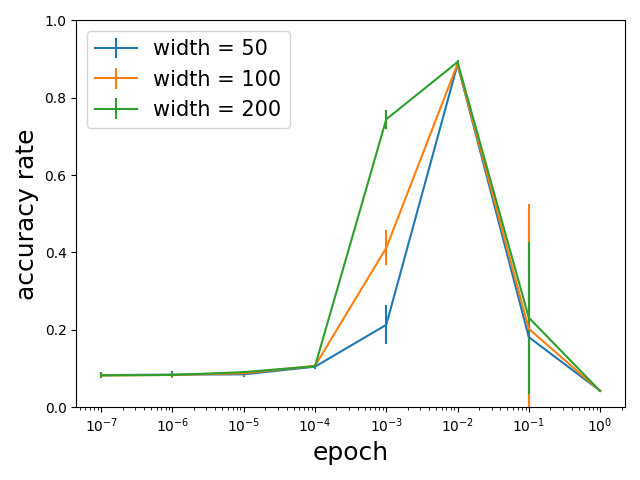} & 
    \includegraphics[width=.3\linewidth]{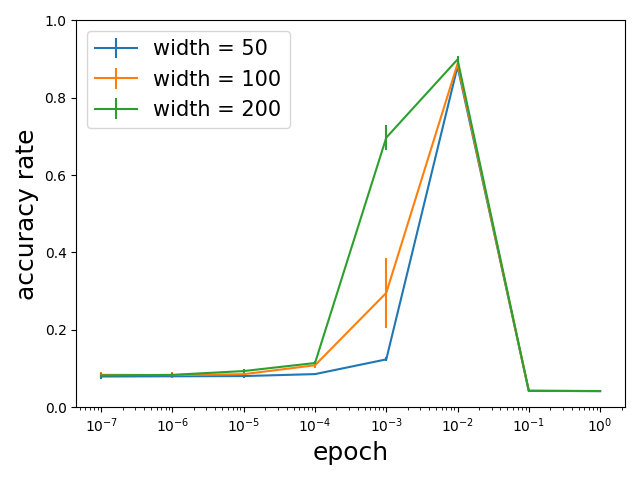} &
    \includegraphics[width=.3\linewidth]{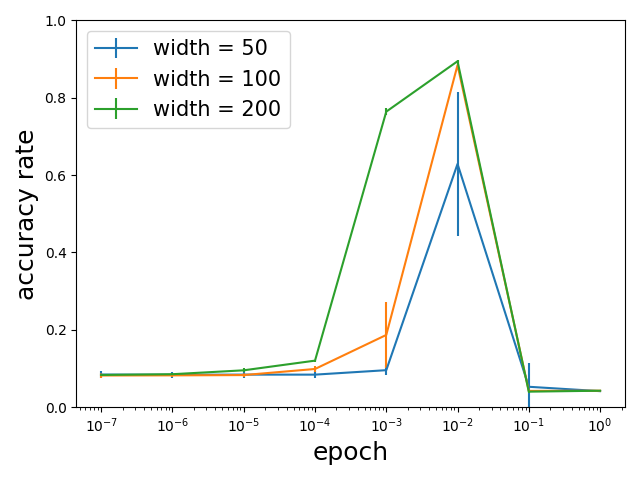} 
     \\
    (a) $g$ of depth 3 & (b) $g$ of depth 6 & (c) $g$ of depth 8\\
    \includegraphics[width=.3\linewidth]{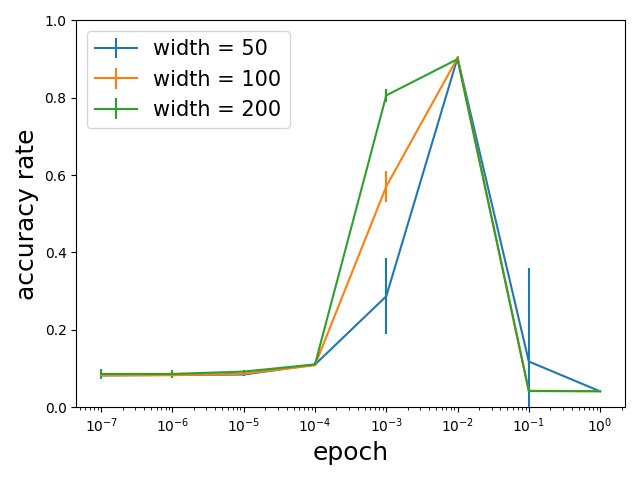} & 
    \includegraphics[width=.3\linewidth]{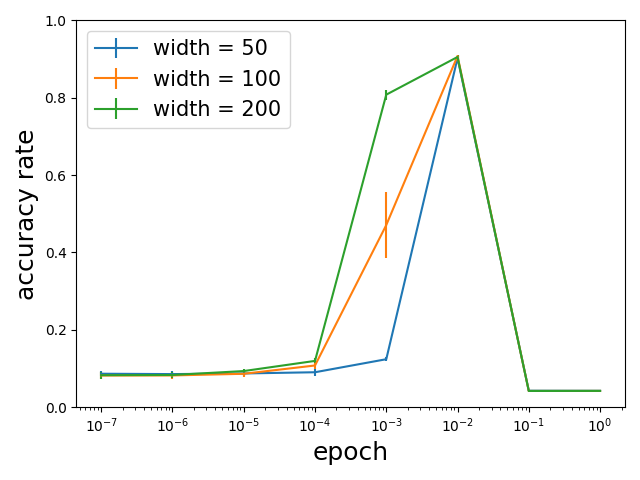} &
    \includegraphics[width=.3\linewidth]{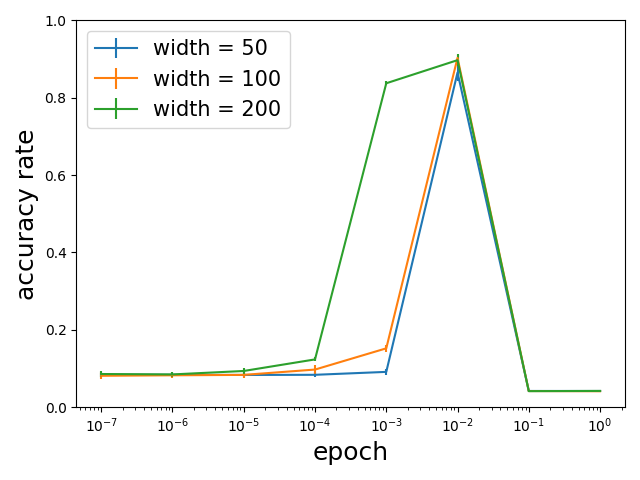} 
    \\
    (d) $g$ of depth 3 & (e) $g$ of depth 6 & (f) $g$ of depth 8\\
    \end{tabular}
    \caption{
    {Sensitivity experiment. We compare the performance of hypernetworks with an implicit network of different widths and depths after $40$ epochs, when varying the learning rate. The x-axis specifies the value of the learning rate and the y-axis specifies the averaged accuracy rate at test time. (a-c) Results on MNIST and (d-f) Results on CIFAR10. In the first column, $g$'s depth is $3$, in the second, it is $6$ and in the third, it is $8$.}}
    \label{fig:varLr}
\end{figure*}


\subsection{Training wide networks with a large learning rate} Remark~1 (main text) states that one is able to train wide networks with a learning rate $\mu = o(n)$. To validate this remark, we trained shallow networks of varying width $n \in \{10^2,10^3,10^4,2.5\cdot 10^5\}$ with learning rate $\mu = \sqrt{n}$ on MNIST. As can be seen in Fig.~\ref{fig:lr}, training those networks is possible despite the very large learning rate. In fact, we observe that the accuracy rate and loss improve as we increase the width of the network.



\begin{figure}[ht]
    \centering
    \begin{tabular}{cc}
\includegraphics[width=0.35\linewidth]{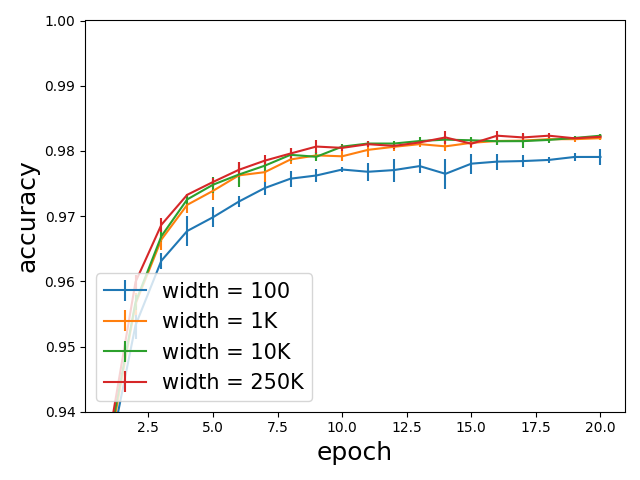} &  
\includegraphics[width=0.35\linewidth]{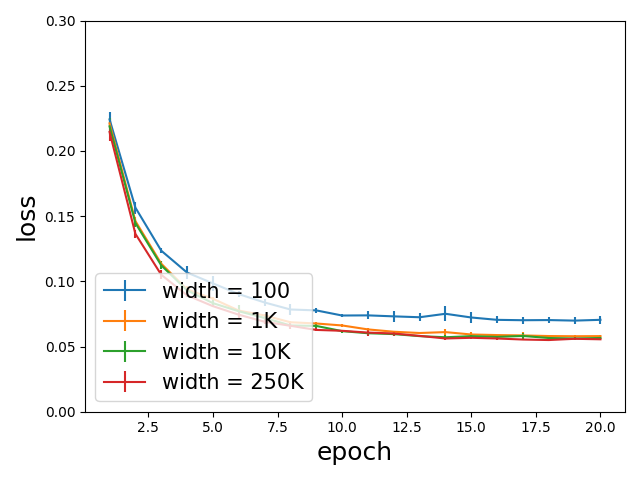} \\
      (a) & (b) 
    \end{tabular}
    \caption{{\bf Results of training wide networks with a large learning rate.} The y-axis is the {\bf (a)} accuracy rate or {\bf (b)} the average loss at test time. We vary the width $n \in \{10^2,10^3,10^4,2.5\cdot 10^5\}$ and take the learning rate to be $\sqrt{n}$.}
    \label{fig:lr}
\end{figure}

\section{Correlation Functions}

Correlation functions are products of general high order tensors representing high order derivatives of a networks output with respect to the weights. 
In~\cite{Dyer2020Asymptotics} a conjecture is posed on the order of magnitude of general correlation functions involving high order derivative tensors, which arise when analysing the dynamics of gradient descent. 
Roughly speaking, given inputs $\{x_i\}_{i=1}^r$, the outputs of a neural network $f(x_1;w),...,f(x_r;w) \in \mathbb{R}$ with normally distributed parameters $w\in \mathbb{R}^N$, correlation functions takes the form:
\begin{equation}
\sum_{\eta_{k_0},...,\eta_{k_r} \in [N]} \prod_{j=1}^{r}\Gamma_{\eta_{k_j+1},...,\eta_{k_{j+1}}}(x_j)
\end{equation}
where 
\begin{equation}
\Gamma_{\eta_{1},...,\eta_{k}}(x_j) := \frac{\partial^k f(x_j;w)}{\partial w_{\eta_1}...\partial w_{\eta_k}}
\end{equation}
For instance, the following are two examples of correlation functions,
\begin{equation}
f(x_1;w) \cdot \frac{\partial f(x_2;w)}{\partial w_{\mu_1}},
\frac{\partial^2 f(x_1;w)}{\partial w_{\mu_1} \partial w_{\mu_2}} \cdot \frac{\partial f(x_2;w)}{\partial w_{\mu_1}}
\end{equation}
Computing the expected value of these correlation functions involve keeping track of various moments of normally distributed weights along paths, as done in recent finite width correction works~\cite{Hanin2020Finite,littwin2020random}.  
\cite{Dyer2020Asymptotics} employ the Feynman diagram to efficiently compute the expected values (order of magnitude) of general correlation functions, albeit at the cost of only being provably accurate for deep linear, or shallow ReLU networks. 
In this work, we analyze the asymptotic behaviour correlation functions of the form: 
\begin{equation}\label{terms}
\begin{aligned}
\mathcal{T}^r(x_0,...,x_r) &:= \sum_{\eta_{k_0}...\eta_{k_r}\in [N]} \Gamma_{\eta_{k_1},...,\eta_{k_r}}(x_0)\prod_{j=1}^{r}\Gamma_{\eta_{k_j}}(x_j) \\
 &= \bigg\langle \nabla^{(r)}_w f(x_0), \bigotimes^{r}_{j=1}\nabla_w f(x_j) \bigg\rangle
 \end{aligned}
 \end{equation}
where $\nabla^{(r)}_w f(x_0)$ is a rank $r$ tensor, representing the $r$'th derivative of the output, and $\bigotimes^{r}_{j=1}\nabla_w f(x_j)$ denotes outer products of the gradients for different examples. 
terms of the form in Eq.~\ref{terms} represent high order terms in the multivariate Taylor expansion of outputs, and are, therefore, relevant for the full understanding of training dynamics. As a consequence of Thm.~\ref{thm:hyper}, we prove that $\mathcal{T}^r(x_0,...,x_r) \sim 1/n^{\max(r-1,0)}$ for vanilla neural networks, where $n$ is the width of the network.
As we have shown in Sec.~3, terms of the form in Eq.~\ref{terms} represent high order terms in the multivariate Taylor expansion of outputs, and are, therefore, relevant for the full understanding of training dynamics. As a consequence of Thm.~\ref{thm:hyper}, we prove that $\mathcal{T}^r(x_0,...,x_r) \sim 1/n^{\max(r-1,0)}$ for vanilla neural networks, where $n$ is the width of the network.

This result is a partial solution to the open problem suggested by~\cite{Dyer2020Asymptotics}. In their paper, they conjecture the asymptotic behaviour of general correlation functions, and predict an upper bound on the asymptotic behaviour of terms of the form in Eq.~\ref{terms} in the order of $\mathcal{O}(1/n)$. Our results therefore proves a stronger version of the conjecture, while giving the exact behaviour as a function of width.

\newpage

\section{Proofs of the Main Results}
\noindent{\bf Terminology and Notations\quad }
Throughout the appendix, we denote by $A \otimes B$ and $A \odot B$ the outer and Hadamard products of the tensors $A$ and $B$ (resp.). When considering the outer products of a sequence of tensors $\{A_i\}^{k}_{i=1}$, we denote, $\bigotimes^{k}_{i=1} A_i = A_1 \otimes \dots \otimes A_k$. We denote by $\sgn(x) := x/\vert x\vert$ the sign function. The notation $X_n \sim a_n$ states that $X_n/a_n$ converges in distribution to some non-zero random variable $X$. 
A convenient property of this notation is that it satisfies: $X_n \cdot Y_n \sim a_n \cdot b_n$ when $X_n\sim a_n$ and $Y_n \sim b_n$.
Throughout the paper, we will make use of sequential limits and denote $n_k,\dots,n_1 \to \infty$ to express that $n_1$ tend to infinity, then $n_2$, and so on. For a given sequence of random variable $\{X_n\}^{\infty}_{n=1}$, we denote by $X_n\stackrel{d}{\longrightarrow} X$ ($X_n\stackrel{p}{\longrightarrow} X$), when $X_n$ converges in distribution (probability) to a random variable $X$.
\subsection{Useful Lemmas}

\begin{lemma}\label{lem:indicator}
Let $X_n\stackrel{d}{\longrightarrow} X$. Then, $\sgn(X_n) \stackrel{d}{\longrightarrow} \sgn(X)$.
\end{lemma}

\begin{proof}
We have: 
\begin{equation}
\lim_{n\to \infty}\mathbb{P}[\sgn(X_n) = 1] = \lim_{n\to \infty}\mathbb{P}[X_n \geq 0] = \mathbb{P}[X \geq 0] = \mathbb{P}[\sgn(X) = 1]    
\end{equation}
Hence, $\sgn(X_n)$ converges in distribution to $\sgn(X)$.
\end{proof}

\subsection{Main Technical Lemma}



In this section, we prove Lem.~\ref{lem:d}, which is the main technical lemma that enables us proving Thm.~\ref{thm:hyper}. Let $f(x;w)$ be a neural network with $H$ outputs $\{f^d(x;w)\}^{H}_{d=1}$. We would like to estimate the order of magnitude of the following expression:
\begin{equation}\label{eq:T}
\mathcal{T}^{\boldsymbol{l,i,d}}_{n,i,d} :=
\left\langle \frac{\partial^k f^d(x_i;w)}{\partial W^{l_1} \dots \partial W^{l_k} } , \bigotimes^{k}_{t=1}\frac{ \partial f^{d_1}(x_{i_t};w)}{\partial W^{l_t}}  \right\rangle
\end{equation}
where $\boldsymbol{d} = (d_1,\dots,d_k)$, $\boldsymbol{i}=(i_1,\dots,i_k)$ and $\boldsymbol{l} = (l_1,\dots,l_k)$. For simplicity, when, $i_1=\dots=i_k=j$, we denote: $\mathcal{T}^{\boldsymbol{l,d}}_{n,i,j,d} := \mathcal{T}^{\boldsymbol{l,i,d}}_{n,i,d}$ and $\mathcal{T}^{\boldsymbol{l}}_{n,i,j,d} := \mathcal{T}^{\boldsymbol{l,i,d}}_{n,i,d}$ when  $d_1=\dots=d_k=d$ as well. 

To estimate the order of magnitude of the expression in Eq.~\ref{eq:T}, we provide an explicit expression for $\frac{\partial^k f^d(x_i;w)}{\partial W^{l_1} \dots \partial W^{l_k} }$. First, we note that for any $w$, such that, $f^d(x_i;w)$ is $k$ times continuously differentiable at $w$, for any set $\boldsymbol{l} := \{l_1,\dots,l_k\}$, we have:
\begin{equation}
\frac{\partial^k f^d(x_i;w)}{\partial W^{l_1} \dots \partial W^{l_k} } = \frac{\partial^k f^d(x_i;w)}{\partial W^{l'_1} \dots \partial W^{l'_k} }
\end{equation}
where the set $\boldsymbol{l}' := \{l'_1,\dots,l'_k\}$ is an ordered version of $\boldsymbol{l}$, i.e., the two sets consist of the same elements but $l'_1 < \dots < l'_k$. In addition, we notice that for any multi-set $\boldsymbol{l}$, such that, $l_i = l_j$ for some $i\neq j$, then, 
\begin{equation}
\frac{\partial^k f^d(x_i;w)}{\partial W^{l_1} \dots \partial W^{l_k} } = 0    
\end{equation}
since $f^d(x_i;w)$ is a neural network with a piece-wise linear activation function. Therefore, with no loss of generality, we consider $\boldsymbol{l} = \{l_1,\dots,l_k\}$, such that, $l_1<\dots<l_k$.  It holds that:
\begin{equation}\label{der}
\frac{\partial^k f^d(x_i;w)}{\partial W^{l_1} \dots \partial W^{l_k} } = \frac{1}{\sqrt{n_{l_1-1}}}q^{l_1-1}_{i,d} \otimes \mathcal{A}^{l_1\rightarrow l_2}_{i,d}
\end{equation}
where $\mathcal{A}^{l_1\rightarrow l_2}_{i,d}$ is a $2k-1$ tensor, defined as follows:
\begin{equation}
\mathcal{A}^{l_j\rightarrow l_{j+1}}_{i,d} = 
\begin{cases}  
\frac{1}{\sqrt{n_{l_{j+1}-1}}}C^{l_j\rightarrow l_{j+1}}_{i,d} \otimes \mathcal{A}^{l_{j+1}\rightarrow l_{j+2}}_{i,d} & 1<j<k-1\\
\frac{1}{\sqrt{n_{l_{k}-1}}}C^{l_{k-1}\rightarrow l_{k}}_{i,d} \otimes C^{l_{k}\rightarrow L}_{i,d} & j=k-1\\
\end{cases} 
\end{equation}
where:
\begin{equation}\label{eq:c}
C^{l_j\rightarrow l_{j+1}}_{i,d} = \begin{cases}  
\sqrt{2}Z^{l_{j+1}-1}_{i,d}P_{i,d}^{l_j\rightarrow l_{j+1} - 1} & l_{j+1} \neq L\\
P_{i,d}^{l_j\rightarrow L} & else
\end{cases}
\end{equation}
and:
\begin{equation}
P_i^{u\rightarrow v} = \prod_{l = u}^{v-1}(\sqrt{\frac{2}{n_{l}}}W^{l+1}Z^{l}_i) \textnormal{ and } Z^l_i = \textnormal{diag} (\dot\sigma(y^{l}(x_i)))
\end{equation}
The individual gradients can be expressed using:
\begin{equation}\label{grad}
\frac{ \partial f^{d_j}_w(x_{i_j})}{\partial W^{l_j}}= \frac{q^{l_j-1}_{i_j,d_j} \otimes C^{l_j\rightarrow L}_{i_j,d_j}}{\sqrt{n_{l_j-1}}}
\end{equation}
Note that the following holds for any $u<v<h\leq L$:
\begin{equation}\label{eq:factor}
C^{u\rightarrow h}_{i,d} = C^{v\rightarrow h}\frac{W^v}{\sqrt{n_{v-1}}}C^{u\rightarrow v}_{i,d} \textnormal{ and } C^{u\rightarrow L}_{i,d} = C^{v-1\rightarrow L}_{i,d} P^{u\rightarrow v-1}_{i,d}
\end{equation}
In the following, given the sets $\boldsymbol{l} = \{l_1,\dots,l_k\}$, $\boldsymbol{i}=\{i_1,\dots,i_k\}$ and $\boldsymbol{d} = \{d_1,\dots,d_k\}$, we derive the limit of $\mathcal{T}^{\boldsymbol{l,i,d}}_{n,i,d}$ using elementary tensor algebra. By Eqs.~\ref{grad} and~\ref{der}, we see that:
\begin{equation}\label{eq:order}
\begin{aligned}
\mathcal{T}^{\boldsymbol{l,i,d}}_{n,i,d}
&= \Big\langle \bigotimes^{k}_{t=1}\frac{ \partial f^{d_t}(x_{i_t};w)}{\partial W^{l_t}},\frac{q^{l_1-1}_{i,d}}{\sqrt{n_{l_1-1}}} \otimes  \frac{C^{l_1\rightarrow l_{2}}_{i,d}}{\sqrt{n_{l_{2}-1}}} \otimes ...\otimes  \frac{C^{l_{r-1} \rightarrow l_k}_{i,d}}{\sqrt{n_{l_k-1}}} \otimes C^{l_k\rightarrow L}_{i,d}\Big\rangle
\\
&=\frac{1}{n_{l_1-1}}\left\langle q^{l_1-1}_{i,d} , q^{l_1-1}_{i_1,d_1}  \right \rangle \cdot \left\langle C_{i_k,d_k}^{l_k\rightarrow L} , C_{i,d}^{l_k\rightarrow L}  \right \rangle \prod_{j=1}^{k-1}\left\langle \frac{C_{i_{j},d_{j}}^{l_j\rightarrow L}\otimes q_{i_{j+1},d_{j+1}}^{l_{j+1}-1}}{n_{l_{j+1}-1}}, C_{i,d}^{l_j\rightarrow l_{j+1}}  \right \rangle
\end{aligned}
\end{equation}

We recall the analysis of~\cite{Yang2019TensorPI} showing that in the infinite width limit, with $n = \min(n_1\dots,n_{L-1}) \to \infty$, every pre-activation $y^l(x)$ of $f(x;w)$ at hidden layer $l\in [L]$ has all its coordinates tending to i.i.d. centered Gaussian processes of covariance $\Sigma^l(x,x') : \mathbb{R}^{n_0} \times \mathbb{R}^{n_0} \to \mathbb{R}$ defined recursively as follows:
\begin{equation}
\begin{aligned}
\Sigma^0(x,x') &= x^{\top} x', \\
\Lambda^l(x,x') &= 
\begin{bmatrix}
\Sigma^{l-1}(x,x) &
\Sigma^{l-1}(x,x') \\
\Sigma^{l-1}(x',x) &
\Sigma^{l-1}(x',x')
\end{bmatrix}
\in \mathbb{R}^{2\times 2},\\
\Sigma^l(x,x') &= \mathbb{E}_{(u,v)\sim \mathcal{N}(0,\Lambda^{l-1}(x,x'))}[\sigma(u)\sigma(v)] 
\end{aligned}
\end{equation}
In addition, we define the derivative covariance as follows:
\begin{equation}
\dot{\Sigma}^l(x,x') = \mathbb{E}_{(u,v)\sim \mathcal{N}(0,\Lambda^{l-1}(x,x'))}[\dot{\sigma}(u)\dot{\sigma}(v)] 
\end{equation}
when considering $x = x_i$ and $x' = x_j$ from the training set, we simply write $\Sigma^l_{i,j} := \Sigma^l(x_i,x_j)$ and $\dot{\Sigma}^l_{i,j} = \dot{\Sigma}^l(x_i,x_j)$.

\begin{lemma}\label{lem:arora1}
The following holds: 
\begin{enumerate}
    \item For $n_{L-1},\dots,n_1\to \infty$, we have: $P_{i,d_1}^{u \rightarrow L}(P_{j,d_2}^{u \rightarrow L})^{\top} \stackrel{d}{\longrightarrow} \prod_{l=u}^{L-1}\dot{\Sigma}_{i,j}^l\delta_{d_1=d_2}$.
    \item For $n_{v},\dots,n_{1}\to \infty$, we have: $\frac{(q_i^{v})^{\top} q_j^v}{n_v} \stackrel{d}{\longrightarrow} \Sigma_{i,j}^v$.
\end{enumerate}
Here, $\delta_{T}$ is an indicator that returns $1$ if $T$ is true and $0$ otherwise.
\end{lemma}

\proof{
See~\cite{conf/nips/AroraDH0SW19}.
}

\begin{lemma}\label{lem:d}
Let $k \geq 0$ and sets $\boldsymbol{l}=\{l_1,\dots,l_k\}$, $\boldsymbol{i}=\{i_1,\dots,i_k\}$ and $\boldsymbol{d}=\{d_1,\dots,d_k\}$. We have:
\begin{equation}
n^{\max(k-1,0)} \cdot \mathcal{T}^{\boldsymbol{l,i,d}}_{n,i,d} \stackrel{d}{\longrightarrow}
\begin{cases}
\delta_{\boldsymbol{d}}\cdot \prod_{j=1}^{k-1}\mathcal{G}_j & k > 1\\
const & k = 1 \\
\end{cases}
\end{equation}
as $n\to \infty$. Here, $\mathcal{G}_1,...,\mathcal{G}_{k-1}$ are centered Gaussian variables with finite, non-zero variances, and $\delta_{\boldsymbol{d}} := \delta(d_1 = ... = d_k = d)$.
\end{lemma}

\begin{proof}
The case $k=0$ is trivial. Let $k\geq 1$. By Eq.~\ref{eq:order}, it holds that:
\begin{equation}
\begin{aligned}
&n^{k-1}\mathcal{T}^{\boldsymbol{l,i,d}}_{n,i,d}\\ 
=& n^{k-1}\frac{\Big\langle q^{l_1-1}_{i,d} , q^{l_1-1}_{i_1,d_1}  \Big \rangle \Big\langle C_{i_k,d_k}^{l_k\rightarrow L} , C_{i,d}^{l_k\rightarrow L}  \Big \rangle}{n} \cdot  \prod_{j=1}^{k-1}\left\langle \frac{C_{i_j,d_j}^{l_j\rightarrow L}\otimes q_{i_{j+1},d_{j+1}}^{l_{j+1}-1}}{n}, C_{i,d}^{l_j\rightarrow l_{j+1}}  \right \rangle \\
=&  \frac{\Big\langle q^{l_1-1}_{i,d} , q^{l_1-1}_{i_1,d_1}  \Big \rangle \Big\langle C_{i_k,d_k}^{l_k\rightarrow L} , C_{i,d}^{l_k\rightarrow L}  \Big \rangle}{n}   \cdot \prod_{j=1}^{k-1} \left\langle C_{i_j,d_j}^{l_j\rightarrow L}\otimes q_{i_{j+1},d_{j+1}}^{l_{j+1}-1}, C_{i,d}^{l_j\rightarrow l_{j+1}}  \right \rangle
\end{aligned}
\end{equation}
Note that intermediate activations do not depend on the index $d_j$, and so we remove the dependency on $d_j$ in the relevant terms. Next, by applying Lem.~\ref{lem:arora1},
\begin{equation}\label{eq:main}
\frac{\Big\langle q^{l_1-1}_{i} , q^{l_1-1}_{i_1}  \Big \rangle \Big\langle C_{i_k,d_k}^{l_k\rightarrow L} , C_{i,d}^{l_k\rightarrow L}  \Big \rangle}{n} 
\stackrel{d}{\longrightarrow} \Sigma_{i,i_1}^{l_1-1}\left(\prod_{j=l_k}^{L}\dot{\Sigma}_{i,i_k}^{l_j}\right)\delta_{\boldsymbol{d}}
\end{equation}
Expanding the second term using Eq.~\ref{eq:factor}:
\begin{equation}\label{eq:22}
\begin{aligned}
&\Big\langle C_{i_j,d_j}^{l_j\rightarrow L}\otimes q_{i_{j+1}}^{l_{j+1}-1}, C_{i,d}^{l_j\rightarrow i_{j+1}}  \Big \rangle \\
&= C_{i_j,d_j}^{l_j\rightarrow L}C_{i}^{l_j\rightarrow i_{j+1}}q_{i_{j+1}}^{l_{j+1}-1} \\
&= C_{i_j,d_j}^{l_{j+1}-1\rightarrow L}P_{i_j}^{l_j\rightarrow l_{j+1}-1} (P_{i}^{l_j\rightarrow l_{j+1}-1})^{\top}  \sqrt{2} \cdot Z_{i}^{l_{j+1}-1}q_{i_{j+1}}^{l_{j+1}-1}\\
&=\sqrt{2}\cdot \Big\langle C_{i_j,d_j}^{l_{j+1}-1\rightarrow L}\otimes (Z_i^{l_{j+1}-1}q_{i_{j+1}}^{l_{j+1}-1}), P_{i_j}^{l_j\rightarrow l_{j+1}-1} (P_{i}^{l_j\rightarrow l_{j+1}-1})^{\top} \Big \rangle \\
&=\sqrt{2}\cdot C_{i_j,d_j}^{l_{j+1}-1\rightarrow L}P_{i_j}^{l_j\rightarrow l_{j+1}-1} (P_{i}^{l_j\rightarrow l_{j+1}-1})^{\top} Z_i^{l_{j+1}-1}q_{i_{j+1}}^{l_{j+1}-1}\\
&= \xi_j.
\end{aligned}
\end{equation} 
The above expression is fully implementable in a \emph{Tensor Program} (see \cite{Yang2019TensorPI,Yang2019ScalingLO}), and approaches a GP as width tend to infinity. In other words:

\begin{equation}
\xi_j \stackrel{d}{\longrightarrow} \hat{\mathcal{G}}_j
\end{equation}
and denoting $\boldsymbol{\xi} = [\xi_1,...,\xi_{k-1}]$, and $\hat{\boldsymbol{\mathcal{G}}} = [\hat{\mathcal{G}}_1,...,\hat{\mathcal{G}}_{k-1}]$, it holds using the multivariate Central Limit theorem:
\begin{equation}
\boldsymbol{\xi} \stackrel{d}{\longrightarrow} \hat{\boldsymbol{\mathcal{G}}}
\end{equation}
Using the Mann-Wald theorem~\cite{mann1943} (where we take the mapping as the product pooling of $\boldsymbol{\xi}$), we have that:
\begin{equation}
\prod_{j=1}^{k-1}\xi_j \stackrel{d}{\longrightarrow} \prod_{j=1}^{k-1}\hat{\mathcal{G}}_j
\end{equation}
Finally, by Slutsky's theorem,
\begin{equation}
\begin{aligned}
n^{k-1}\mathcal{T}^{\boldsymbol{l,i,d}}_{n,i,d} \stackrel{d}{\longrightarrow} \Sigma_{i,i_1}^{l_1-1}\left(\prod_{j=l_k}^{L}\dot{\Sigma}_{i,i_k}^{l_j}\right) \cdot \left(\prod^{k-1}_{j=1} \hat{\mathcal{G}}_j \right) \cdot \delta_{\boldsymbol{d}}.
\end{aligned}
\end{equation}
Assigning $\mathcal{G}_j:= \Sigma_{i,i_1}^{l_1-1}\left(\prod_{j=l_k}^{L}\dot{\Sigma}_{i,i_k}^{l_j}\right)\hat{\mathcal{G}}_j$ completes the proof.
\end{proof}

\subsection{Proof of Thm.~\ref{thm:hyper}}

Since we assume that $g$ is a finite neural network, i.e., $m_l < \infty$ for all $l \in [H]$, throughout the proofs with no loss of generality we assume that $m_1=\dots=m_H=1$. 

\begin{lemma}\label{lem:h1}
Let $h(u;w) = g(z;f(x;w))$ be a hypernetwork. We have: 
\begin{equation}
\begin{aligned}
\mathcal{K}^{(r)}_{i,j} =& \sum_{\substack{\alpha_1 + \dots + \alpha_{H} = r\\ \alpha_1,\dots,\alpha_{H} \geq 0}} \frac{r!}{\alpha_1 ! \cdots \alpha_{H}! }\cdot z_i \cdot \left[ \prod^{H-1}_{j=1} \dot\phi(g^{j}_i) \right]  \cdot \prod^{H}_{d=1} \left\langle \nabla^{(\alpha_{d})}_{w} f^{d}_i , (\nabla_w h_j )^{\alpha_d} \right\rangle 
\end{aligned}
\end{equation}
\end{lemma}

\begin{proof}
By the higher order product rule and the fact that the second derivative of a piece-wise linear function is $0$ everywhere: 
\begin{equation}
\begin{aligned}
\nabla^{(r)}_{w} h_i = \sum_{\substack{\alpha_1 + \dots + \alpha_{H} = r\\ \alpha_1,\dots,\alpha_{H} \geq 0}} \frac{r!}{\alpha_1 ! \cdots \alpha_{H}! }  \cdot z_i \cdot \nabla^{(\alpha_{H})}_{w} f^{H}_i \bigotimes^{H-1}_{d=1} D_{H-d} 
\end{aligned}
\end{equation}
where 
\begin{equation}
D_d := \dot\phi(g^{d}_i) \cdot \nabla^{(\alpha_{d})}_{w} f^{d}_i
\end{equation}
In addition, by elementary tensor algebra, we have:
\begin{equation}
\begin{aligned}
\mathcal{K}^{(r)}_{i,j} =&\langle \nabla^{(r)}_{w} h_i ,(\nabla_w h_j )^r \rangle \\
=& \sum_{\substack{\alpha_1 + \dots + \alpha_{H} = r\\ \alpha_1,\dots,\alpha_{H} \geq 0}} \frac{r!}{\alpha_1 ! \cdots \alpha_{H}! }  z_i \cdot\left\langle   \nabla^{(\alpha_{H})}_{w} f^{H}_i \cdot \bigotimes^{H-1}_{d=1} D_{H-d} , (\nabla_w h_j )^r \right\rangle\\
=& \sum_{\substack{\alpha_1 + \dots + \alpha_{H} = r\\ \alpha_1,\dots,\alpha_{H} \geq 0}} \frac{r!}{\alpha_1 ! \cdots \alpha_{H}! } z_i\cdot \left\langle  \nabla^{(\alpha_{H})}_w f^{H}_i, (\nabla_w h_j )^{\alpha_{H}} \right\rangle\\
&\quad\quad\quad\quad\quad\quad \cdot \prod^{H-1}_{d=1} \left\langle \dot\phi(g^{H-d}_i) \cdot \nabla^{(\alpha_{H-d})}_{w} f^{H-d}_i , (\nabla_w h_j )^{\alpha_{H-d}} \right\rangle \\
=& \sum_{\substack{\alpha_1 + \dots + \alpha_{H} = r\\ \alpha_1,\dots,\alpha_{H} \geq 0}} \frac{r!}{\alpha_1 ! \cdots  \alpha_{H}! }\cdot z_i \cdot \left[ \prod^{H-1}_{d=1} \dot\phi(g^{d}_i) \right] \cdot \prod^{H}_{d=1} \left\langle \nabla^{(\alpha_{d})}_{w} f^{d}_i , (\nabla_w h_j )^{\alpha_{d}} \right\rangle 
\end{aligned}
\end{equation}
\end{proof}

\begin{lemma}\label{lem:h2}
Let $h(u;w) = g(z;f(x;w))$ be a hypernetwork. In addition, let,
\begin{equation}
\forall d\in [H]:~~h^d_j := 
a^{d-1}_j \prod^{H-d}_{t=1} f^{H-t+1}_j \cdot \dot\phi(g^{H-t}_j )
\end{equation}
We have:
\begin{equation}
\begin{aligned}
\left\langle \nabla^{(\alpha_d)}f^d_i , (\nabla_w h_j)^{\alpha_{d}} \right\rangle = \sum_{\boldsymbol{l}\in [L]^{\alpha_d}} \sum_{\boldsymbol{d} \in [H]^{\alpha_d}} \left(\prod^{\alpha_d}_{k=1} h^{d_k}_j\right) \cdot  \mathcal{T}^{\boldsymbol{l,d}}_{n,i,j,d} 
\end{aligned}
\end{equation}
\end{lemma}

\begin{proof}
We have:
\begin{equation}
\begin{aligned}
\left\langle \nabla^{(\alpha_d)}f^d_i , (\nabla_w h_j )^{\alpha_{d}} \right\rangle = \sum_{\boldsymbol{l} \in [L]^{\alpha_d}}  \left\langle \frac{\partial^{\alpha_d} f^d_i }{\partial W^{l_1} \dots \partial W^{l_{\alpha_d}} }, \bigotimes^{\alpha_d}_{k=1} \frac{ \partial h_j }{\partial W^{l_k}} \right\rangle
\end{aligned}
\end{equation}
By the product rule:
\begin{equation}
\begin{aligned}
\frac{\partial h_j }{\partial W^{l_k}} 
=& 
\sum^{H}_{d=1} \left[\prod^{H-d}_{t=1} f^{H-t+1}_j \cdot \dot\phi(g^{H-t}_j) \right] \cdot \frac{\partial f^d_j }{\partial W^{l_k}} \cdot a^{d-1}_j =\sum^{H}_{d=1} h^d_j  \cdot \frac{\partial f^d_j}{\partial W^{l_k}} 
\end{aligned}
\end{equation}
Hence, 
\begin{equation}
\begin{aligned}
\bigotimes^{\alpha_d}_{k=1}\frac{ \partial h_j}{\partial W^{l_k}}
= \sum_{\boldsymbol{d} \in [H]^{\alpha_d}} \left(\prod^{\alpha_d}_{k=1} h^{d_k}_j\right)  \bigotimes^{\alpha_d}_{k=1}\frac{ \partial f^{d_k}_j}{\partial W^{l_k}}
\end{aligned}
\end{equation}
In particular, 
\begin{equation}
\begin{aligned}
\left\langle \nabla^{(\alpha_d)}f^d_i , (\nabla_w h_j)^{\alpha_{d}} \right\rangle = \sum_{\boldsymbol{l}\in [L]^{\alpha_d}} \sum_{\boldsymbol{d} \in [H]^{\alpha_d}} \left(\prod^{\alpha_d}_{k=1} h^{d_k}_j\right) \cdot  \mathcal{T}^{\boldsymbol{l,d}}_{n,i,j,d}
\end{aligned}
\end{equation}
\end{proof}


\inithypernet*

\begin{proof}
Throughout the proof, in order to derive certain limits of various sequences of random variables, we implicitly make use of the Mann-Wald theorem~\cite{mann1943}. For simplicity, oftentimes, we will avoid explicitly stating when this theorem is applied. As a general note, the repeated argument is as follows: terms, such as, $n^{\max(\alpha_d-1,0)} \cdot \mathcal{T}^{\boldsymbol{l,d}}_{n,i,j,d}$, $\mathcal{Q}^{\boldsymbol{d}}_{n,j}$, $g^d_i$, etc', (see below) can be expressed as continuous mappings of jointly convergent random variables. Hence, they jointly converge, and continuous mappings over them converge as well. 

By Lems.~\ref{lem:h1} and~\ref{lem:h2}, we have: 
\begin{equation}
\begin{aligned}
\mathcal{K}^{(r)}_{i,j} 
= \sum_{\substack{\alpha_1 + \dots + \alpha_{H} = r\\ \alpha_1,\dots,\alpha_{H} \geq 0}} \frac{r!}{\alpha_1 ! \cdots  \alpha_{H}! } \cdot z_i \cdot \left[\prod^{H-1}_{d=1} \dot{\phi}(g^d_i) \right]  \cdot \prod^{H}_{d=1} \sum_{\boldsymbol{l}\in [H]^{\alpha_d}}  \sum_{\boldsymbol{d} \in [H]^{\alpha_d}} \mathcal{Q}^{\boldsymbol{d}}_{n,j} \cdot \mathcal{T}^{\boldsymbol{l,d}}_{n,i,j,d}
\end{aligned}
\end{equation}
where $\mathcal{Q}^{\boldsymbol{d}}_{n,j} := \left(\prod^{\alpha_d}_{k=1} h^{d_k}_j\right)$. By the Mann-Wald theorem~\cite{mann1943}$, g^d_i$ converges to some random variable $\mathcal{U}^d_i$. If $\dot\phi$ is a continuous function, then $\dot{\phi}(g^d_i)$ converges to $\dot\phi(\mathcal{U}^d_i)$. If $\phi$ is the ReLU activation function, by Lem.~\ref{lem:indicator}, $\dot{\phi}(g^d_i) = \sgn(g^d_i)$ converges to $\sgn(\mathcal{U}^d_i)$ in distribution. We notice that $\mathcal{Q}^{\boldsymbol{d}}_{n,j}$ converges in distribution to some random variable $\mathcal{Q}^{\boldsymbol{d}}_j$. 



The proof is divided into two cases: $H=1$ and $H>1$.

\paragraph{Case $H=1$:} First, we note that for $H=1$ and $d\in [H]$ (i.e., $d=1$), we have: 
\begin{equation}
\begin{aligned}
h^d_j &= a^{d-1}_j \cdot \prod^{H-d}_{t=1} f^{H-t+1}_j \cdot \dot{\sigma}(g^{H-t}_j) = a^{0}_j = z_j \\
\end{aligned}
\end{equation}
In addition, $\prod^{H-1}_{d=1} \dot{\sigma}(g^d_i) = 1$ as it is an empty product. Therefore, we can rewrite:
\begin{equation}
\mathcal{K}^{(r)}_{i,j} = z_i \cdot z^{r}_j \sum_{\boldsymbol{l}\in [H]^{r}}  \sum_{\boldsymbol{d} \in [H]^{r}} \mathcal{T}^{\boldsymbol{l,d}}_{n,i,j,d}
\end{equation}
By Lem.~\ref{lem:d}, for $r=1$, the above tends to a constant as $n\to \infty$. For $r>1$, $n^{r-1} \cdot \mathcal{T}^{\boldsymbol{l,d}}_{n,i,j,d}$ converges in distribution to zero for all $\boldsymbol{d} \neq (d,\dots,d)$ and converges to a non-constant random variable $\mathcal{T}^{\boldsymbol{l}}_{i,j,d}$ otherwise. Hence, by the Mann-Wald theorem~\cite{mann1943},
\begin{equation}
n^{r-1} \cdot \mathcal{K}^{(r)}_{i,j} \stackrel{d}{\longrightarrow} z_i \cdot z^r_j \sum_{\boldsymbol{l} \in [H]^r} \mathcal{T}^{\boldsymbol{l}}_{i,j,d}
\end{equation} 
which is a non-zero random variable. 

\paragraph{Case $H>1$:} 

By Lem.~\ref{lem:d}, $n^{\alpha_d-1} \cdot \mathcal{T}^{\boldsymbol{l,d}}_{n,i,j,d}$ converges in distribution to zero for all $\boldsymbol{d} \neq (d,\dots,d)$. Therefore, in these cases, by Slutsky's theorem, $n^{\alpha_d-1} \cdot \mathcal{Q}^{\boldsymbol{d}}_{n,j} \cdot \mathcal{T}^{\boldsymbol{l,d}}_{n,i,j,d}$ converges to zero in distribution. On the other hand, for each $\boldsymbol{l} \in [H]^{\alpha_d}$, $d\in [H]$ and $\boldsymbol{d} = (d,\dots,d)$, by Lem.~\ref{lem:d}, we have: 
\begin{equation}
n^{\alpha_d-1} \cdot \mathcal{Q}^{\boldsymbol{d}}_{n,j} \cdot \mathcal{T}^{\boldsymbol{l}}_{n,i,j,d} \stackrel{d}{\longrightarrow} \mathcal{Q}^{\boldsymbol{d}}_j \cdot \mathcal{T}^{\boldsymbol{l}}_{i,j,d}
\end{equation}
In particular, 
\begin{equation}
n^{\max(\alpha_d-1,0)} \sum_{\boldsymbol{l}\in [H]^{\alpha_d}} \sum_{\boldsymbol{d} \in [H]^{\alpha_d}} \cdot \mathcal{Q}^{\boldsymbol{d}}_{n,j} \cdot \mathcal{T}^{\boldsymbol{l}}_{n,i,j,d} \stackrel{d}{\longrightarrow} \sum_{\boldsymbol{l}\in [H]^{\alpha_d}} \sum_{d \in [H]} \mathcal{Q}^{d}_j \cdot \mathcal{T}^{\boldsymbol{l}}_{i,j,d}
\end{equation}
Consider the case where $r \geq H$. In this case, for any $\alpha_1,\dots,\alpha_H$, such that, there are $t>1$ indices $i \in [H]$, such that, $\alpha_i=0$. The following random variable converges in distribution:
\begin{equation}
X_n := n^{r-(H-t)} \cdot \prod^{H}_{d=1} \sum_{\boldsymbol{l}\in [H]^{\alpha_d}}  \sum_{\boldsymbol{d} \in [H]^{\alpha_d}} \mathcal{Q}^{\boldsymbol{d}}_{n,j} \cdot  \mathcal{T}^{\boldsymbol{l,d}}_{n,i,j,d}
\end{equation}
Therefore, by Slutsky's theorem:
\begin{equation}
n^{r-H} \cdot \prod^{H}_{d=1} \sum_{\boldsymbol{l}\in [H]^{\alpha_d}}  \sum_{\boldsymbol{d} \in [H]^{\alpha_d}} \mathcal{Q}^{\boldsymbol{d}}_{n,j} \cdot \mathcal{T}^{\boldsymbol{l,d}}_{n,i,j,d}  = n^{-t} \cdot X_n \stackrel{d}{\longrightarrow} 0
\end{equation}
We have:
\begin{equation}
\begin{aligned}
&n^{r-H}  \cdot \left\langle \nabla^{(r)}_{w} h_i , (\nabla_w h_j )^{r} \right\rangle \\
=& n^{r-H} \sum_{\substack{\alpha_1 + \dots + \alpha_{H} = r\\ \alpha_1,\dots,\alpha_{H} \geq 0}} \frac{r!}{\alpha_1 ! \cdots  \alpha_{H}! } \cdot z_i \cdot \left[\prod^{H-1}_{d=1} \dot{\sigma}(g^d_i) \right]  \cdot \prod^{H}_{d=1} \sum_{\boldsymbol{l}\in [H]^{\alpha_d}}  \sum_{\boldsymbol{d} \in [H]^{\alpha_d}} \left(\prod^{\alpha_d}_{k=1} h^{d_k}_j\right)  \mathcal{T}^{\boldsymbol{l,d}}_{n,i,j,d} \\
=& \sum_{\substack{\alpha_1 + \dots + \alpha_{H} = r\\ \alpha_1,\dots,\alpha_{H} \geq 0}} \frac{r!}{\alpha_1 ! \cdots  \alpha_{H}! } \cdot z_i \cdot \left[\prod^{H-1}_{d=1} \dot{\sigma}(g^d_i) \right]  \cdot \prod^{H}_{d=1} n^{\alpha_d-1} \sum_{\boldsymbol{l}\in [H]^{\alpha_d}}  \sum_{\boldsymbol{d} \in [H]^{\alpha_d}} \mathcal{Q}^{\boldsymbol{d}}_{n,j} \cdot  \mathcal{T}^{\boldsymbol{l,d}}_{n,i,j,d} \\
\stackrel{d}{\longrightarrow}& \sum_{\substack{\alpha_1 + \dots + \alpha_{H} = r\\ \alpha_1,\dots,\alpha_{H} \geq 1}} \frac{r!}{\alpha_1 ! \cdots  \alpha_{H}! } \cdot z_i \cdot \left[\prod^{H-1}_{d=1} \sgn(\mathcal{U}^d_i) \right]  \cdot \prod^{H}_{d=1} \sum_{\boldsymbol{l}\in [H]^{\alpha_d}}  \mathcal{Q}^d_j \cdot \mathcal{T}^{\boldsymbol{l}}_{i,j,d}
\end{aligned}
\end{equation}
which is a non-constant random variable.

Next, we consider the case when $r \leq H$. 
By Lem.~\ref{lem:d}, for any $\alpha_d \geq 2$, the term $\mathcal{T}^{\boldsymbol{l,d}}_{n,i,j,d}$ tends to zero as $n\to \infty$. In addition, $\mathcal{Q}^{\boldsymbol{d}}_{n,j}$ converges in distribution. Therefore, for any $\alpha_d \geq 2$, we have:
\begin{equation}
\sum_{\boldsymbol{l} \in [L]^{\alpha_d}} \sum_{\boldsymbol{d} \in [H]^{\alpha_d}} \mathcal{Q}^{\boldsymbol{d}}_{n,j}\cdot \mathcal{T}^{\boldsymbol{l,d}}_{n,i,j,d} \stackrel{d}{\longrightarrow} 0
\end{equation}
Hence, for any $\alpha_1,\dots,\alpha_H \geq 0$, such that, there is at least one $\alpha_d \geq 2$, we have:
\begin{equation}
\prod^{H}_{d=1} \sum_{\boldsymbol{l} \in [L]^{\alpha_d}} \sum_{\boldsymbol{d} \in [H]^{\alpha_d}} \mathcal{Q}^{\boldsymbol{d}}_{n,j}\cdot \mathcal{T}^{\boldsymbol{l,d}}_{n,i,j,d} \stackrel{d}{\longrightarrow} 0
\end{equation}
On the other hand, for any $0\leq \alpha_1,\dots,\alpha_H \leq 1$, the terms $\{\mathcal{T}^{\boldsymbol{l,i,d}}_{n,i,d}\}$, $\{g^d_i\}$ and $\{\mathcal{Q}^{\boldsymbol{d}}_{n,j}\}$ converge jointly in distribution to some random variables $\{\mathcal{T}^{\boldsymbol{l,i,d}}_{i,d}\}$, $\{\sgn(\mathcal{U}^d_i)\}$ and $\{\mathcal{Q}^{\boldsymbol{d}}_{j}\}$ as $n \to \infty$. Hence, 
\begin{equation}
\langle \nabla^{(r)}_w h_i, (\nabla_w h_j)^r\rangle \stackrel{d}{\longrightarrow} \sum_{\substack{\alpha_1 + \dots + \alpha_H = r\\ 0\leq \alpha_1,\dots,\alpha_H \leq 1}} r! \cdot \left[\prod^{H-1}_{d=1}\sgn(\mathcal{U}^d_i)\right]\cdot \prod^{H}_{d=1} \sum_{\boldsymbol{l} \in [L]^{\alpha_d}} \sum_{\boldsymbol{d} \in [H]^{\alpha_d}} \mathcal{Q}^{\boldsymbol{d}}_{j}\cdot \mathcal{T}^{\boldsymbol{l,d}}_{i,j,d}
\end{equation}
which is a non-constant random variable.
\end{proof}

\subsection{Proofs of the Results in Sec.~4}


\GP*

\begin{proof} By~\cite{Yang2019TensorPI}, taking the width $n=\min(n_1,...,n_{L-1})$ to infinity, the outputs $V^{d}(x;w) := f^d(x;w)$ are governed by a centered Gaussian process, such that, the entries $V^d_{i,j}(x;w)$, given some input $x$, are independent and identically distributed. Moreover, it holds that:
\begin{equation}
\Big(V^d_{i,j}(x;w),V^d_{i,j}(x';w)\Big) \sim \mathcal{N}\left(0,S^L(x,x')\right).
\end{equation}
with $S^L(x,x')$ as defined in Eq.~\ref{s}.
For the function $h(u;w) = g(z;f(x;w))$, it holds for the first layer:
\begin{equation}
g^1(z;f(x;w)) = \sqrt{\frac{1}{m_0}}V^1(x;w) z
\end{equation}
After taking the limit $n=\min(n_1,...,n_{L-1})$ to infinity, the implicit network $g$ is fed with Gaussian distributed weights.
And so $g^1(z;f(x;w))$ also converges to a Gaussian process, such that:
\begin{equation}
(g^1(z;f(x;w))_i,g^1(z';f(x';w))_i) \sim \mathcal{N}(0,\Lambda^1)
\end{equation}
where:
\begin{equation}
\Lambda^1 = \frac{1}{m_0}\begin{pmatrix}
S^L(x,x) z^{\top} z & S^L(x',x)z'^{\top} z  \\
S^L(x,x') z^{\top} z' & S^L(x',x')z'^{\top} z'  
\end{pmatrix}
\end{equation}
In a similar fashion to the standard feed forward case, the pre-activations $g^l(z;f(x;w))$ converge to Gaussian processes as we let $m = \min(m_1,...,m_{H-1})$ tend to infinity, with a covariance defined recursively:
\begin{equation}
\Sigma^l(u,u') =  \sqrt{2}\E_{(u,v)\sim \mathcal{N}(0,\Lambda^l)}[\sigma(u)\sigma(v)]
\end{equation}
where,
\begin{equation}
\Lambda^l = 
\begin{pmatrix}
S^L(x,x)\cdot \Sigma^{l-1}(u,u) & 
S^L(x',x)\cdot \Sigma^{l-1}(u',u)  \\
S^L(x,x')\cdot \Sigma^{l-1}(u,u') & 
S^L(x',x')\cdot \Sigma^{l-1}(u',u')  
\end{pmatrix}
\end{equation}
and
\begin{equation}
\Sigma^0(z,z') = \frac{1}{m_0}z^\top z'
\end{equation}
proving the claim.
\end{proof}


\compGP*

\begin{proof}
We prove that $\Lambda^{l}(u,u')$ is a function of $S^0(x,x')$ and $\Sigma^0(u,u')$ by induction. First, we note that $\Lambda^{1}(u,u')$ is a function of $S^L(x,x')$ and $\Sigma^0(u,u')$ by definition. By the recursive definition of $S^L(x,x')$, it is a function of $S^0(x,x')$. Therefore, $\Lambda^{1}(u,u')$ can be simply represented as a function of $S^0(x,x')$ and $\Sigma^0(u,u')$. We assume by induction that $\Lambda^{l}(u,u')$ is a function of $S^0(x,x')$ and $\Sigma^0(u,u')$. We would like to show that $\Lambda^{l+1}(u,u')$ is a function of $S^0(x,x')$ and $\Sigma^0(u,u')$. By definition, $\Lambda^{l+1}(u,u')$ is a function of $S^L(x,x')$ and $\Sigma^l(u,u')$. In addition, $\Sigma^l(u,u')$ is a function of $\Lambda^l(u,u')$. Hence, by induction, $\Sigma^l(u,u')$ is simply a function of $S^0(x,x')$ and $\Sigma^0(u,u')$. Since $S^L(x,x')$ is a function of $S^0(x,x')$, we conclude that one can represent $\Lambda^{l+1}(u,u')$ as a function of $S^0(x,x')$ and $\Sigma^0(u,u')$.
\end{proof}


\periodicGP*

\begin{proof}
We note that: 
\begin{equation}
\Sigma^0(p(z),p(z')) = \frac{1}{k} p(z)^{\top} p(z) = \frac{1}{k} \sum^{k}_{i=1} \cos(W^1_i z + b^1_i) \cos(W^1_i z' + b^1_i) 
\end{equation}
By Thm.~1 in~\cite{NIPS2007_3182}, we have:
\begin{equation}
\lim_{k\to \infty} \Sigma^0(p(z),p(z')) = \exp[-\|z-z'\|^2_2/2]/2
\end{equation}
which is a function of $\exp[\|z-z'\|^2_2]$ as desired.
\end{proof}


We make use of the following lemma in the proof of Thm.~\ref{thm:hyperkernel}.

\begin{lemma}\label{lem:arora2}
Recall the parametrization of the implicit network:
\begin{equation}
\begin{cases}  
g^l_i := g^l(z_i;v) = \sqrt{\frac{1}{m_{l-1}}} f^l(x_i;w)\cdot a_i^{l-1} \\
a^l_i := a^l(z_i;v) = \sqrt{2} \cdot \sigma(g^l_i)
\end{cases} 
\textnormal{ and } a^0_i := z_i
\end{equation}
For any pair $u_i = \{u_i\}$, we denote:
\begin{equation}
P_i^{l_1\rightarrow l_2} = \prod_{l = l_1}^{l_2-1}\left(\sqrt{\frac{2}{m_{l}}}V^{l+1}(x_i;w)\cdot Z^{l}(z_i) \right) \textnormal{ and } Z^l(z) = \textnormal{diag} (\dot\sigma(g^{l}(z)))
\end{equation}
It holds that:
\begin{enumerate}
    \item $P_i^{l_1 \rightarrow l_2}(P_j^{l_1 \rightarrow l_2})^{\top} \stackrel{d}{\longrightarrow} \prod_{l=l_1}^{l_2-1}\dot{\Sigma}^l(u_i,u_j)I$.
    \item $\frac{\partial h(u_i,w)}{\partial v}\cdot \frac{\partial^{\top} h(u_j,w)}{\partial v} \stackrel{d}{\longrightarrow} \sum_{l=0}^{H-1}\left(\Sigma^l(u_i,u_j)\prod_{h=l+1}^{H-1}\dot{\Sigma}^l(u_i,u_j)\right)$.
\end{enumerate}
where the limits are taken with respect to $m,n \to \infty$ sequentially.
\end{lemma}

\begin{proof}
We have:
\begin{equation}
\begin{aligned}
&P_i^{l_1 \rightarrow l_2}(P_j^{l_1 \rightarrow l_2})^{\top} \\
=& P_i^{l_1 \rightarrow l_2-1}\frac{2}{m_{l_2-1}}V^{l_2}(x_i;w) 
\cdot Z^{l_2-1}(z_i)Z^{l_2-1}(z_j)V^{l_2}(x_j;w)^\top (P_j^{l_1 \rightarrow l_2-1})^{\top} 
\end{aligned}
\end{equation}
Note that it holds that when $m,n\to \infty$ sequentially, we have:
\begin{equation}
\begin{aligned}
&\frac{2}{m_{l_2-1}}V^{l_2}(x_i;w)\cdot Z^{l_2-1}(z_i)Z^{l_2-1}(z_j)V^{l_2}(x_j;w)^\top \\
\stackrel{d}{\longrightarrow} & \sqrt{2}\E_{(u,v)\sim \mathcal{N}(0,\Lambda^{l_2})}[\dot{\sigma(u)}\dot{\sigma(v)}]I = \dot{\Sigma}^{l_2}(u_i,u_j)I
\end{aligned}
\end{equation}
Applying the above recursively proves the first claim.
Using the first claim, along with the derivation of the neural tangent kernel (see \cite{conf/nips/AroraDH0SW19}) proves the second claim.
\end{proof}


\hyperkernel*

\begin{proof}

Recalling that $v = vec(f(x)) = [vec(V^1),...,vec(V^{H})]$, concatenated into a single vector of length $\sum_{l=0}^{H-1}m_l \cdot m_{l+1}$. The components of the inner matrix $\mathcal{K}^f(x,x')$ are given by:
\begin{equation}
\mathcal{K}^f(x,x')(i,j) = \sum_{l=1}^L \left\langle \frac{\partial v_i(x)}{\partial w^l},\frac{\partial v_j(x')}{\partial w^l} \right\rangle\\
\end{equation}
and it holds that in the infinite width limit, $\mathcal{K}^f(x,x')$ is a diagonal matrix:
\begin{equation}
\mathcal{K}^f(x,x') \stackrel{d}{\longrightarrow} \Theta^f(x,x')\cdot I 
\end{equation}
By letting the widths $n$ and $m$ tend to infinity consecutively, by Lem.~\ref{lem:arora2}, it follows that:
\begin{equation}
\frac{\partial h(u;w)}{\partial v} \cdot \frac{\partial^{\top} h(u';w)}{\partial v} \stackrel{d}{\longrightarrow} \Theta^g(u,u',S^L(x,x'))
\end{equation}
Since $\mathcal{K}^f(x,x') = \frac{\partial f(x;w)}{\partial w} \cdot \frac{\partial^\top f(x';w) }{w}$ converges to the diagonal matrix $\Theta^f(x,x')\cdot I$, the limit of $\mathcal{K}^h(u,u')$ is given by:
\begin{equation}
\begin{aligned}
\mathcal{K}^h(u,u') =& \frac{\partial g(z;f(x;w))}{\partial f(x;w)}\cdot \frac{\partial f(x;w)}{\partial w} \cdot \frac{\partial^\top f(x';w) }{w} \cdot \frac{\partial^{\top} g(z';f(x';w))}{\partial f(x';w)} \\ 
=&  
\frac{\partial h(u;w)}{\partial v} \cdot \frac{\partial f(x;w)}{\partial w} \cdot \frac{\partial^\top f(x';w) }{w} \cdot 
\frac{\partial^{\top} h(u';w)}{\partial v}  \\
\stackrel{d}{\longrightarrow} & \Theta^f(x,x')\cdot \Theta^g(u,u',S^L(x,x'))
\end{aligned}
\end{equation}
where we used the results of Lem.~\ref{lem:arora2}.

Next, we would like to prove that $\frac{\partial \mathcal{K}^h(u,u')}{\partial t}\Big\vert_{t=0} = 0$. For this purpose, we write the derivative explicitly:
\begin{equation}
\begin{aligned}
\frac{\partial \mathcal{K}^h(u,u')}{\partial t} = \frac{\partial h(u;w)}{\partial w} \cdot \frac{\partial}{\partial t}\frac{\partial^\top h(u';w)}{\partial w} 
+ \frac{\partial}{\partial t}  \frac{\partial h(u;w)}{\partial w} \cdot \frac{\partial ^\top h(u';w)}{\partial w} \\
\end{aligned}
\end{equation}
We notice that the two terms are the same up to changing between the inputs $u$ and $u'$. Therefore, with no loss of generality, we can simply prove the convergence of the second term. We have:
\begin{equation}
\begin{aligned}
&\frac{\partial}{\partial t}  \frac{\partial h(u;w)}{\partial w} \cdot \frac{\partial ^\top h(u';w)}{\partial w} \\
=& \left[ \frac{\partial}{\partial t}  \left(\frac{\partial h(u;w)}{\partial f(x;w)} \cdot \frac{\partial f(x;w)}{\partial w}\right) \right] \cdot \frac{\partial ^\top h(u';w)}{\partial w} \\
=& \left[ \frac{\partial h(u;w)}{\partial f(x;w) \partial t} \cdot \frac{\partial f(x;w)}{\partial w} + \frac{\partial h(u;w)}{ \partial f(x;w)} \cdot \frac{\partial f(x;w)}{\partial w \partial t} \right] \cdot \frac{\partial ^\top h(u';w)}{\partial w} \\
=& \frac{\partial h(u;w)}{\partial f(x;w) \partial t} \cdot \frac{\partial f(x;w)}{\partial w} \cdot \frac{\partial ^\top h(u';w)}{\partial w} +  \frac{\partial h(u;w)}{ \partial f(x;w)} \cdot \frac{\partial f(x;w)}{\partial w \partial t} \cdot \frac{\partial ^\top h(u';w)}{\partial w} \\
\end{aligned}
\end{equation}
We analyze each term separately. 

\paragraph{Analyzing the first term} By substituting $\frac{\partial}{\partial t} = -\mu \nabla_w c(w)\frac{\partial^\top}{\partial w} = -\mu \nabla_w c(w)\frac{\partial^\top f}{\partial w}\frac{\partial^\top }{\partial f}$, we have:
\begin{equation}
\begin{aligned}
&\frac{\partial h(u;w)}{\partial f(x;w) \partial t}\cdot \frac{\partial f(x;w)}{\partial w}\cdot \frac{\partial ^\top h(u';w)}{\partial w}\\
&=-\mu\nabla_w c(w)\frac{\partial^\top f(x;w)}{\partial w} \cdot \frac{\partial^2 h(u;w)}{\partial f(x;w)\partial f(x;w)}\cdot \frac{\partial f(x;w)}{\partial w}\cdot \frac{\partial^{\top} f(x';w)}{\partial w}\cdot \frac{\partial^\top h(u';w)}{\partial f(x';w)}\\
&=-\mu\nabla_w c(w)\frac{\partial^\top f(x;w)}{\partial w}\frac{\partial^2 h(u;w)}{\partial f(x;w)\partial f(x;w)}\mathcal{K}^f(x,x')\cdot \frac{\partial^\top h(u';w)}{\partial f(x';w)}\\
&=-\mu\sum^N_{i=1} \frac{\partial \ell(h(u_i;w),y_i)}{\partial h(u_i;w)}\cdot \frac{\partial h(u_i;w)}{\partial f(x;w)}\cdot \mathcal{K}^f(x,x_i)\cdot \frac{\partial^2 h(u;w)}{\partial f(x;w)\partial f(x;w)}\cdot \mathcal{K}^f(x,x')\cdot \frac{\partial^\top h(u';w)}{\partial f(x';w)}
\end{aligned}
\end{equation}
It then follows:
\begin{equation}
\begin{aligned}
&\lim_{n \to \infty} \frac{\partial h(u;w)}{\partial f(x;w) \partial t}\cdot \frac{\partial f(x;w)}{\partial w}\cdot \frac{\partial^\top h(u';w)}{\partial w}\\
=&-\mu\sum^{N}_{i=1}\ell_i \cdot \Theta^f(x,x_i)\cdot \Theta^f(x,x')\lim_{n \to \infty}\frac{\partial h(u_i;w)}{\partial f(x_i;w)}\cdot \frac{\partial^2 h(u;w)}{\partial f(x;w)\partial f(x;w)}\cdot \frac{\partial h(u';w)}{\partial f(x';w)}\\
\end{aligned}
\end{equation}
We notice that:
\begin{equation}
\begin{aligned}
&\lim_{n\to \infty}\frac{\partial^\top h(u_i;w)}{\partial f(x_i;w)}\cdot \frac{\partial^2 h(u;w)}{\partial f(x;w)\partial f(x;w)}\cdot \frac{\partial h(u';w)}{\partial f(x';w)} \\
=&\sum_{l_1,l_2} \lim_{n \to \infty} \left\langle\frac{\partial^2 h(u;w)}{\partial f^{l_1}(x;w)\partial f^{l_2}(x;w)},  \frac{\partial h(u_i;w)}{\partial f^{l_1}(x_i;w)} \otimes \frac{\partial h(u';w)}{\partial f^{l_2}(x';w)} \right\rangle\\
:=& \sum_{l_1,l_2}\mathcal{T}^{l_1,l_2}_{m}(u,u_i,u')\\
\end{aligned}
\end{equation}
We recall that $f^{l}(x;w)$ converges to a GP (as a function of $x$) as $n \to \infty$~\cite{lee2018deep}. Therefore, $\mathcal{T}^{l_1,l_2}_{m}(u,u_i,u')$ are special cases of the terms $\mathcal{T}^{\boldsymbol{l,i,d}}_{n,i,d}$ (see Eq.~\ref{eq:T}) with weights that are distributed according to a GP instead of a normal distribution. In this case, we have: $k=2$, $d=d_1=\dots=d_k=1$, the neural network $f^1$ is replaced with $h$, the weights $W^l$ are translated into $f^l(x;w)$. We recall that the proof of Lem.~\ref{lem:d} showing that $\mathcal{T}^{\boldsymbol{l,i,d}}_{n,i,d} = \mathcal{O}_p(1/n^{k-1})$ is simply based on Lem.~\ref{lem:arora1}. Since Lem.~\ref{lem:arora2} extends Lem.~\ref{lem:arora1} to our case, the proof of Lem.~\ref{lem:d} can be applied to show that $\mathcal{T}^{l_1,l_2}_{m}(u,u_i,u') \sim 1/m$.

\paragraph{Analyzing the second term} We would like to show that for any $m>0$, we have:
\begin{equation}
\frac{\partial h(u;w)}{ \partial f(x;w)} \cdot \frac{\partial f(x;w)}{\partial w \partial t} \cdot \frac{\partial ^\top h(u';w)}{\partial w} \stackrel{d}{\longrightarrow} 0
\end{equation}
as $n\to \infty$. Since $\frac{\partial w}{\partial t} = -\mu \nabla_w c(w)$, we have:
\begin{equation}
\begin{aligned}
&\frac{\partial h(u;w)}{ \partial f(x;w)} \cdot \frac{\partial f(x;w)}{\partial w \partial t} \cdot \frac{\partial ^\top h(u';w)}{\partial w} \\
=& -\mu \cdot \frac{\partial h(u;w)}{ \partial f(x;w)}  \cdot \nabla_w c(w) \cdot \frac{\partial^2 f(x;w)}{\partial w^2} \cdot \frac{\partial ^\top h(u';w)}{\partial w} \\
=& -\mu \cdot \frac{\partial h(u;w)}{ \partial f(x;w)} \cdot \nabla_w c(w) \cdot \frac{\partial^2 f(x;w)}{\partial w^2} \cdot \frac{\partial^{\top} f(x';w)}{\partial w} \cdot \frac{\partial ^\top h(u';w)}{\partial f(x;w)} \\
\end{aligned}
\end{equation}
In addition, we have:
\begin{equation}
\begin{aligned}
\nabla_w c(w) = \sum^{N}_{i=1} \frac{\partial \ell(h(u_i;w),y_i)}{\partial h(u_i;w)} \cdot  \frac{\partial h(u_i;w)}{\partial w} 
\end{aligned}
\end{equation}
We note that $\frac{\partial \ell(h(u_i;w),y_i)}{\partial h(u_i;w)}$ converges in distribution as $m,n\to \infty$. 
Therefore, we can simply analyze the convergence of:
\begin{equation}
\begin{aligned}
\sum^{N}_{i=1}\frac{\partial h(u;w)}{ \partial f(x;w)} \cdot \frac{\partial h(u_i;w)}{\partial w} \cdot \frac{\partial^2 f(x;w)}{\partial w^2} \cdot \frac{\partial^{\top} f(x';w)}{\partial w} \cdot \frac{\partial ^\top h(u';w)}{\partial f(x;w)} \\
\end{aligned}
\end{equation}
Since $N$ is a constant, it is enough to show that each term converges to zero. We have:
\begin{equation}\label{eq:sumfww}
\begin{aligned}
&\frac{\partial h(u;w)}{ \partial f(x;w)} \cdot \frac{\partial h(u_i;w)}{\partial w} \cdot \frac{\partial^2 f(x;w)}{\partial w^2} \cdot \frac{\partial^{\top} f(x';w)}{\partial w} \cdot \frac{\partial ^\top h(u';w)}{\partial f(x;w)} \\
=&\frac{\partial h(u;w)}{ \partial f(x;w)} \cdot \frac{\partial h(u_i;w)}{\partial f(x_i;w)} \cdot \frac{\partial f(x_i;w)}{\partial w} \cdot \frac{\partial^2 f(x;w)}{\partial w^2} \cdot \frac{\partial^{\top} f(x';w)}{\partial w} \cdot \frac{\partial ^\top h(u';w)}{\partial f(x;w)} \\
=&\sum_{l,j,k}\frac{\partial h(u;w)}{ \partial f(x;w)_l} \cdot \frac{\partial h(u_i;w)}{\partial f(x_i;w)_j}\cdot \frac{\partial f(x_i;w)_j}{\partial w} \cdot \frac{\partial^2 f(x;w)_l}{\partial w^2} \cdot \frac{\partial^{\top} f(x';w)_k}{\partial w} \cdot \frac{\partial ^\top h(u';w)}{\partial f(x;w)_k} \\
\end{aligned}
\end{equation}
where $f(x;w)_j$ is the $j$'th output of $f$ over $x$. In addition, the summation is done over the indices of the corresponding tensors. We note that for any $m>0$, the number of indices $l,j,k$ is finite. We would like to show that each term in the sum tends to zero as $n\to \infty$. We can write:
\begin{equation}\label{eq:fww}
\frac{\partial f(x_i;w)_j}{\partial w} \cdot \frac{\partial^2 f(x;w)_l}{\partial w^2} \cdot \frac{\partial^{\top} f(x';w)_k}{\partial w} 
= \left\langle \frac{\partial^2 f(x;w)_l}{\partial w^2},\frac{\partial f(x_i;w)_j}{\partial w} \otimes \frac{\partial f(x';w)_k}{\partial w}  \right\rangle
\end{equation} 
By Lem.~\ref{lem:d}, the term in Eq.~\ref{eq:fww} tends to zero as $n \to \infty$. In addition, it is easy to see that $\frac{\partial h(u;w)}{ \partial f(x;w)_l}$, $ \frac{\partial h(u_i;w)}{\partial f(x_i;w)_j}$ and $\frac{\partial ^\top h(u';w)}{\partial f(x;w)_k} $ converge to some random variables. Therefore, for any fixed $m>0$, the above sum converges to zero as $n\to \infty$.
\end{proof}

\end{document}